%% file: neurips_2026.tex
\documentclass{article}

\usepackage[main,final]{neurips_2026}

\usepackage[utf8]{inputenc}
\usepackage[T1]{fontenc}
\usepackage{url}
\usepackage{booktabs}
\usepackage{amsfonts}
\usepackage{nicefrac}
\usepackage{microtype}
\usepackage{xcolor}
\usepackage{caption}        
\usepackage{graphicx}
\usepackage{subcaption}
\usepackage{float}

\usepackage{amsmath}
\usepackage{amssymb}
\usepackage{mathtools}
\usepackage{amsthm}

\theoremstyle{plain}
\newtheorem{theorem}{Theorem}[section]

\theoremstyle{definition}

\theoremstyle{remark}

\usepackage{enumitem}
\usepackage{multirow}
\usepackage{pifont}
\usepackage{makecell}
\usepackage{tabularx}
\usepackage[toc,page,header]{appendix}
\usepackage{minitoc}

\usepackage{hyperref}
\usepackage[capitalize,noabbrev]{cleveref}

\makeatletter
\AtBeginDocument{%
  \let\hyper@natlinkstart\@gobble
  \let\hyper@natlinkend\@empty
  \def\hyper@natlinkbreak#1#2{#1}%
}
\makeatother

\usepackage[disable,textsize=tiny]{todonotes}

\title{Steer2Edit: From Activation Steering to Component-Level Editing}

\author{%
  Chung-En Sun \quad
  Ge Yan\thanks{Equal contribution.} \quad
  Zimo Wang\footnotemark[1] \quad
  Tsui-Wei Weng \\
  University of California, San Diego \\
  \texttt{\{cesun, geyan, zimowang, lweng\}@ucsd.edu}
}

\hypersetup{%
  pdftitle={Steer2Edit: From Activation Steering to Component-Level Editing},
  pdfauthor={Chung-En Sun, Ge Yan, Zimo Wang, Tsui-Wei Weng}
}

\begin{document}

\maketitle

\doparttoc
\faketableofcontents

\begin{abstract}
Steering methods influence Large Language Model behavior by identifying semantic directions in hidden representations, and are typically realized through inference-time activation interventions that apply a fixed, global modification to the model's internal states. While effective, such interventions often induce unfavorable attribute--utility trade-offs under strong control, as they ignore the fact that many behaviors are governed by a small and heterogeneous subset of model components. To alleviate the trade-offs, we propose \textsc{\textbf{Steer2Edit}}, a theoretically grounded, training-free framework that transforms steering vectors from inference-time control signals into diagnostic signals for component-level rank-1 weight editing. Instead of uniformly injecting a steering direction during generation, \textsc{\textbf{Steer2Edit}} selectively redistributes behavioral influence across individual attention heads and MLP neurons, yielding interpretable edits that preserve the standard forward pass and remain compatible with optimized parallel inference. Across multiple tasks including safety alignment, truthfulness promotion, and reasoning efficiency, \textsc{\textbf{Steer2Edit}} consistently achieves a superior attribute--utility trade-offs: at matched downstream performance, it outperforms the strongest steering baseline by 17.2\% in safety, 9.8\% in truthfulness, and 12.2\% in reasoning length reduction. Overall, \textsc{\textbf{Steer2Edit}} provides a principled bridge between representation steering and weight editing by translating steering signals into \emph{interpretable}, \emph{training-free} parameter updates. Our code is available at: \textsf{{\small \href{https://github.com/Trustworthy-ML-Lab/Steer2Edit}{https://github.com/Trustworthy-ML-Lab/Steer2Edit}}}
\end{abstract}

\input{sec/Introduction}
\input{sec/Preliminary}
\input{sec/Method}
\input{sec/Experiments}
\input{sec/RelatedWorks}
\input{sec/Conclusion}

\section*{Acknowledgement}
C. Sun, G. Yan, and T.-W. Weng are partially supported by National Science Foundation under Grant No. 2313105, 2430539, and Intel Rising Star Faculty Award. The authors would like to thank the National Research Platform (NRP) at UCSD for computational support and anonymous reviewers for valuable feedback.

\bibliography{example_paper}
\bibliographystyle{plainnat}

\newpage
\appendix
\addcontentsline{toc}{section}{Appendix}
\part{}
\parttoc

\input{sec/Appendix}


\end{document}

%% file: sec/Introduction.tex
\section{Introduction}

Large Language Models (LLMs) have demonstrated strong capabilities across a wide range of tasks, including multi-step reasoning \citep{deepseek,sun2025refine}, code generation \citep{codex}, and planning \citep{reacts}. As these models are increasingly deployed in real-world settings, there is growing interest in \emph{controlling} specific model behaviors without retraining or fully fine-tuning the model.

A prominent line of recent work addresses this goal through \emph{representation steering} \citep{representation_engineering, steering, refusal, reflctrl, unlearning}. These methods identify a \emph{steering vector} in the model’s hidden representation space that correlates with a target attribute, and then intervene at inference time by adding this vector to intermediate activations. Compared to full fine-tuning, steering-based methods offer a lightweight way to adapt a model to different behaviors.

Despite their flexibility, activation-space steering methods suffer from two fundamental limitations. \emph{First}, steering applies a \emph{global} modification to the hidden representation. While such interventions can induce the target behavior, they treat all tokens and internal components uniformly, regardless of how the behavior is realized within the model. Empirical and mechanistic studies show that many behaviors are governed by a small and heterogeneous subset of model components \citep{inductionheads,rome,safetyneurons,safetyheads,drefa}, typically involving specific attention heads or MLP neurons, while most components are unrelated to the target attribute. By ignoring this internal structure, global steering can interfere with unrelated features, resulting in unfavorable trade-offs between the controlled attribute and downstream performance. \emph{Second}, activation-space steering relies on inference-time modification of intermediate activations. This departs from the standard forward pass assumed by modern optimized inference and training systems, which typically require fixed computation graphs. As a result, activation-level interventions complicate integration with standard deployment, parallel inference, and fine-tuning pipelines. While this limitation can in principle be mitigated with additional system engineering, activation steering remains an inference-time control mechanism whose effects are tied to the decoding process, rather than being encoded in the model parameters.

These limitations motivate a different perspective: instead of treating steering vectors as control signals injected into the forward pass, can we reinterpret them as \emph{diagnostic signals} that reveal how a target behavior is distributed across model components? If so, can this information be used to selectively modify the components that genuinely govern the behavior, while avoiding unnecessary interference that degrades utility?

To answer these questions, we introduce \textsc{\textbf{Steer2Edit}}, a theoretically grounded framework that converts steering vectors into \emph{component-level weight edits}. In \textsc{\textbf{Steer2Edit}}, a steering vector is treated as a diagnostic signal that reveals which attention heads and MLP neurons align with a target behavior, and to what extent. Guided by this signal, the method applies coordinated rank-1 updates to individual components, selectively amplifying or suppressing their contributions along the steering direction, rather than inducing a global activation shift. By redistributing behavioral influence at the component level, \textsc{\textbf{Steer2Edit}} enables more precise behavioral control and more favorable attribute--utility trade-offs, while yielding interpretable component-level edits. The resulting procedure is closed-form, requires no fine-tuning or iterative optimization, and produces a standalone edited model that operates under the standard forward pass and remains compatible with existing parallel training and inference pipelines.
\vspace{-5pt}
\paragraph{Contributions.}
\begin{itemize}[leftmargin=1.2em]
    \item We propose \textsc{\textbf{Steer2Edit}}, the first theoretically grounded framework that translates steering vectors into component-level rank-1 weight edits, requiring no fine-tuning and admitting a closed-form, single-step solution.
    \item We show that \textsc{\textbf{Steer2Edit}} consistently achieves a superior attribute--utility trade-off across diverse behavioral control settings: at matched downstream performance, it outperforms the strongest steering baseline by \textbf{17.2\%} in safety, \textbf{9.8\%} in truthfulness, and \textbf{12.2\%} in reasoning length reduction on average.
    \item We show that \textsc{\textbf{Steer2Edit}} produces a standalone edited model that preserves the original architecture, while offering fine-grained interpretability into which components govern specific behaviors and how these behaviors are distributed across the network.

\end{itemize}

\begin{figure*}[!t]
    \centering
    \includegraphics[width=\linewidth]{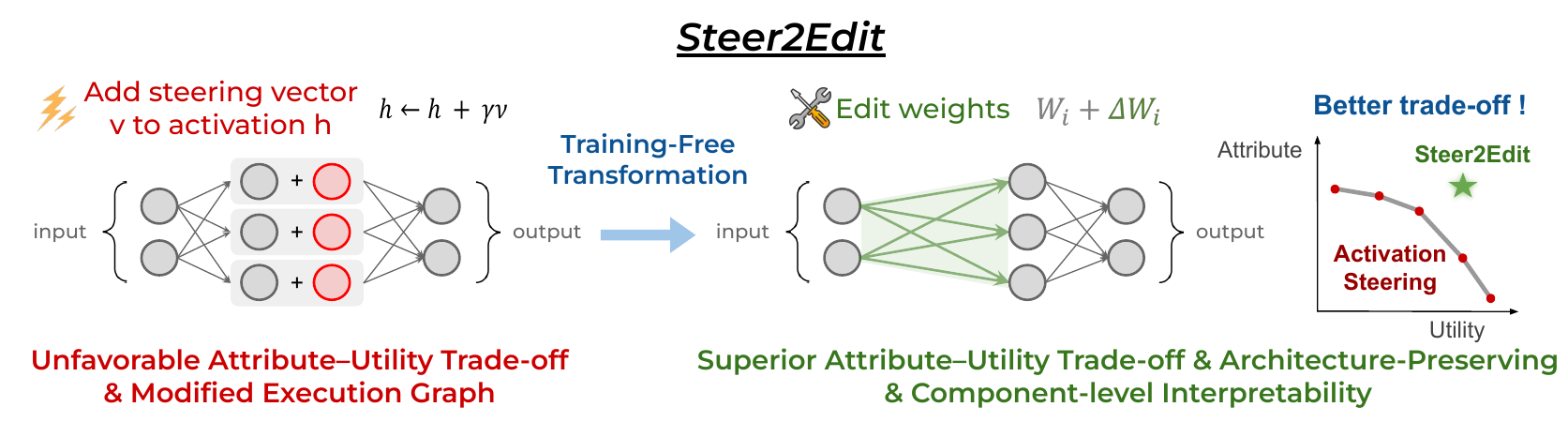}
    \vspace{-20pt}
    \caption{Overview of \textsc{Steer2Edit}. Instead of changing intermediate activations during inference (left), \textbf{\textsc{Steer2Edit}} transforms the steering signal into weight edits (right), achieving better attribute--utility trade-offs including \textit{safety}, \textit{truthfulness}, and \textit{efficient reasoning} settings as shown in Sec~\ref{sec:experiments}.}
    \label{fig:overview1}
    \vspace{-10pt}
\end{figure*}

\begin{table}[t]
\caption{Comparison of behavioral control methods for LLMs.}
\label{tab:method_compare}
\centering
\footnotesize
\setlength{\tabcolsep}{3pt}
\begin{tabular}{lccccc}
\toprule
Method & \makecell{Training-\\free} & \makecell{Superior\\attribute--utility trade-off} & \makecell{Zero inference\\overhead} & \makecell{Standard\\forward pass} & \makecell{Component\\interpretability} \\
\midrule
RepEngineering \citep{representation_engineering}    & \ding{51} & \ding{55} & \ding{55} & \ding{55} & \ding{55} \\
CAA \citep{panickssery2023steering}                  & \ding{51} & \ding{55} & \ding{55} & \ding{55} & \ding{55} \\
ODESteer \citep{zhao2026odesteer}                    & \ding{55} & \ding{55} & \ding{55} & \ding{55} & \ding{55} \\
AlphaEdit \citep{fang2024alphaedit}                  & \ding{51} & \ding{55} & \ding{51} & \ding{51} & \ding{55} \\
SFT / LoRA                                           & \ding{55} & \ding{55} & \ding{51} & \ding{51} & \ding{55} \\
\midrule
\textsc{Steer2Edit} (Ours)                           & \ding{51} & \ding{51} & \ding{51} & \ding{51} & \ding{51} \\
\bottomrule
\end{tabular}
\vspace{-10pt}
\end{table}

%% file: sec/Preliminary.tex
\section{Preliminary}
\label{sec:steer2edit_preliminaries}

In this section, we define the steering vectors used throughout and specify the editable weight components targeted by \textsc{\textbf{Steer2Edit}}.


\paragraph{Steering vector.}
Steering vectors are commonly used in activation steering, where a semantic direction in hidden representations is added to the residual stream at inference time to control model behavior. Such vectors can be constructed in various ways. For simplicity, we adopt a mean-difference construction in this work.

We consider a standard pre-normalization Transformer with $L$ layers and residual-stream dimension $d$. At each layer $\ell$, the attention and MLP blocks each produce an additive update $\delta^{b}_\ell \in \mathbb{R}^d$, $b \in \{\text{attn}, \text{mlp}\}$, that is written into the residual stream. Let $\mathcal{X}$ denote a set of prompts. For each prompt $x \in \mathcal{X}$, the model generates a completion $y$, which is classified as exhibiting or not exhibiting the target attribute, yielding $\mathcal{Y}_{\text{pos}}$ and $\mathcal{Y}_{\text{neg}}$. At token position $t$ of $y$, let
$\delta^b_\ell(y,t) \in \mathbb{R}^d$
denote the block-$b$ output at layer $\ell$ before it is added to the residual stream. Aggregating over token positions $\mathcal{T}_y$ and averaging over generations, we define
\[
\overline{\delta}^{b}_{\ell, a}
=
\frac{1}{|\mathcal{Y}_a|}
\sum_{y \in \mathcal{Y}_a}
\frac{1}{|\mathcal{T}_y|}
\sum_{t \in \mathcal{T}_y}
\delta^{b}_{\ell}(y, t),
\;\;
a \in \{\text{pos}, \text{neg}\}.
\]
The steering vector at layer $\ell$ and block $b$ is given by the mean difference
\begin{equation}
\label{eq:steering_vector}
v^{b}_{\ell}
=
\overline{\delta}^{b}_{\ell,\text{pos}}
-
\overline{\delta}^{b}_{\ell,\text{neg}}
\in \mathbb{R}^{d}.
\end{equation}

\paragraph{Editable weight components and notation.}
We focus on linear weight components whose outputs produce the block activations
from which the steering vectors in Eq.\eqref{eq:steering_vector} are extracted.
Specifically, for each layer $\ell$ and block type
$b \in \{\text{attn}, \text{mlp}\}$, we consider linear maps whose outputs contribute
to the block output $\delta^{b}_{\ell}$ \emph{before} it is written into the residual
stream.

Concretely, these editable components include:
(i) the output projection ($o\_proj$) of an individual attention head in the
attention block, and
(ii) the down-projection ($down\_proj$) associated with a single neuron in the
MLP block.
We denote any such component generically by
\[
W_i \in \mathbb{R}^{d_{\text{out}} \times d_{\text{in}}},
\qquad d_{\text{out}} = d,
\]
where the index $i$ implicitly identifies a specific layer $\ell$, block type $b$,
and component.

For an input activation $h_i \in \mathbb{R}^{d_{\text{in}}}$ to component $W_i$,
the component output $W_i h_i$ lies in the same residual-stream space $\mathbb{R}^d$
as the corresponding steering vector $v^{b}_{\ell}$.
Accordingly, in the subsequent analysis we associate each editable component $W_i$
with the steering vector extracted from the same layer and block, and write this
vector simply as $v_i$.

%% file: sec/Method.tex
\section{Steer2Edit}
\label{sec:steer2edit}

In this section, we introduce \textsc{\textbf{Steer2Edit}}, a principled framework for component-level weight editing based on given steering vectors. We parameterize each edit as a rank-1 update and derive its form by decomposing the problem into three parts: (i) identifying the output-space direction that preserves semantic invariance, (ii) the input-space direction that aligns the edit with the component’s intrinsic semantic contribution, and (iii) the scalar magnitude that allocates edit strength under a global regularization budget.

\begin{figure*}[!t]
    \centering
    \includegraphics[width=0.75\linewidth]{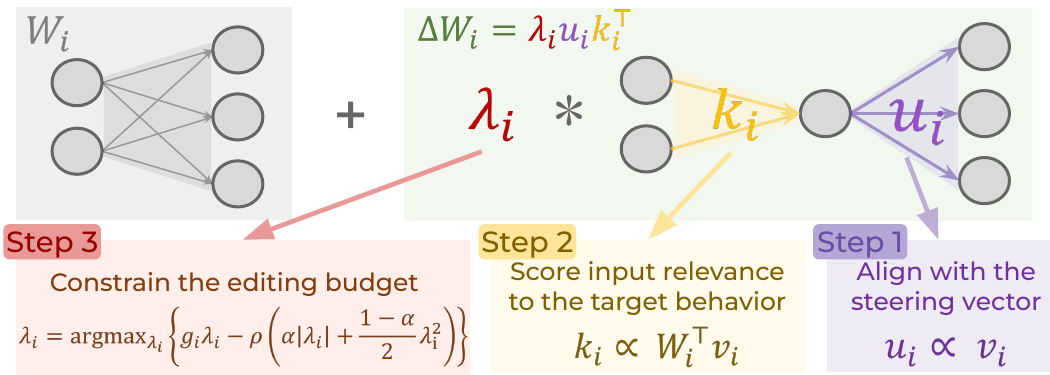}
    \vspace{-5pt}
    \caption{Closed-form derivation of \textsc{Steer2Edit}. Each component receives a rank-1 update \(\Delta W_i = \lambda_i u_i k_i^\top\) determined in three steps: (1)~the output direction \(u_i\) is selected to ensure edits act along the target behavioral axis; (2)~the input direction \(k_i\) is derived to activate the edit selectively on relevant inputs; (3)~the magnitude \(\lambda_i\) is allocated to balance sparsity and edit strength across all components. All three quantities admit closed-form solutions, making the full procedure training-free.}
    \label{fig:overview}
\end{figure*}

\subsection{Assumption and Setting}

For each editable component \(W_i \in \mathbb{R}^{d_{\text{out}} \times d_{\text{in}}}\),
we assume the existence of a steering vector
\(v_i \in \mathbb{R}^{d_{\text{out}}}\) extracted from the same representation space
into which \(W_i\) writes (e.g., the hidden state after an attention or MLP block).
Thus, the output dimension of \(W_i\) matches that of \(v_i\), and both lie in a
common semantic space.
Note that \textsc{\textbf{Steer2Edit}} is agnostic to how \(v_i\) is obtained.

Our goal is to modify each component \(W_i\) so that the resulting update
\(\Delta W_i\) alters the model’s behavior along the semantic direction \(v_i\).
Because the steering signal specifies a single direction in representation space,
we model each edit as a rank-1 perturbation, which is the minimal modification that
can inject a directional effect into a linear map.
Accordingly, we parameterize the edit as
\[
\Delta W_i = \lambda_i\, u_i k_i^{\top},
\]
where \(u_i \in \mathbb{R}^{d_{\text{out}}}\) is an output-space direction,
\(k_i \in \mathbb{R}^{d_{\text{in}}}\) is an input-space direction, and
\(\lambda_i \in \mathbb{R}\) is a scalar magnitude controlling the strength of the edit.

Given an input activation \(h_i\), let
\(
o_i := W_i h_i \in \mathbb{R}^{d_{\text{out}}}
\)
denote the original output of component \(W_i\).
After applying the rank-1 edit \(\Delta W_i = \lambda_i u_i k_i^{\top}\),
the edited output is
\(
\tilde{o}_i
:= (W_i + \Delta W_i)h_i
= o_i + \Delta o_i,
\)
where the induced output shift is
\vspace{-5pt}
\begin{equation}
    \Delta o_i
:= \Delta W_i h_i
= \lambda_i\, u_i (k_i^{\top} h_i).
\label{e:delta o}
\end{equation}

The three quantities \((u_i, k_i, \lambda_i)\) play distinct roles in the edit.
The output-space direction \(u_i\) specifies the semantic direction affected by
the edit, the input-space direction $k_i$ determines which inputs activate the
edit through the inner product $k_i^\top h_i$, and the scalar magnitude \(\lambda_i\) controls how strongly each component
is modified.

We derive these quantities in a sequential order.
We first determine \(u_i\), then \(k_i\), and finally solve for
\(\{\lambda_i\}_{i=1}^n\).
As shown in the following sections, this ordering is without loss of generality:
the optimal choice of \(u_i\) depends only on the steering vector \(v_i\);
the choice of \(k_i\) depends on \(u_i\) and local properties of the component
\(W_i\); and once the geometric directions are fixed, the magnitudes
\(\{\lambda_i\}\) can be optimized independently.

Hence, we derive the three components of the edit in the following order:
\vspace{-5pt}
\begin{enumerate}
  \item the output-space direction \(u_i\) in Section \ref{sec: solving output space direction};
  \item the input-space direction \(k_i\) in Section \ref{sec: solving input space direction};
  \item the scalar magnitude \(\lambda_i\) in Section \ref{sec: solving edit magnitudes}.
\end{enumerate}
\vspace{-5pt}

Throughout the following derivations, we identify only the directions of
$u_i$ and $k_i$; their scale and sign are absorbed into the scalar
coefficients $\lambda_i$, which are determined in the last step.

\subsection{Step 1: Solving for the Output-space Direction \texorpdfstring{\(u_i\)}{u(i)}}
\label{sec: solving output space direction}

The vector \(u_i\) determines the \emph{direction} of the output shift
\(\Delta o_i\).
Because the steering signal specifies a single semantic direction \(v_i\),
we require that the edit modifies the component’s output \emph{only} along this
direction and introduces no change in any orthogonal subspace.

Formally, recall that
\(
\Delta o_i = \Delta W_i h_i.
\)
Semantic invariance requires that, for any input \(h_i\), the output shift
\(\Delta o_i\) has zero projection onto any direction orthogonal to \(v_i\):
\begin{equation}
    z^{\top} \Delta o_i
=
z^{\top} \Delta W_i h_i
= 0
\qquad
\forall\, h_i,\; \forall\, z \perp v_i.
\label{e:output shift}
\end{equation}
This constraint directly restricts the output-space direction \(u_i\) of the edit.
\begin{theorem}[Output-space direction under semantic invariance]
\label{thm:ui}
Let \(v_i \neq 0\), and let \(\Delta W_i = \lambda_i\, u_i k_i^{\top}\) be a
rank-1 edit with \(\Delta W_i \neq 0\).
If for all \(h_i\) and all \(z \perp v_i\) we have
\[
z^{\top} \Delta W_i h_i = 0,
\]
then the output-space direction \(u_i\) must be collinear with \(v_i\), i.e.,
\[
u_i \in \operatorname{span}\{v_i\}.
\]
\end{theorem}

We defer the proof to Appendix~\ref{app:proof-ui}. Theorem~\ref{thm:ui} shows that enforcing semantic invariance in Eq.\eqref{e:output shift} uniquely constrains
the output-space direction: any valid rank-1 update must lie entirely
along the steering direction \(v_i\). Importantly, this result is independent of
both the input-space direction \(k_i\) and the scalar magnitude \(\lambda_i\).
This separation justifies solving for \(u_i\) first in our derivation. We therefore adopt the canonical normalized choice
\[
u_i = \hat{v}_i := \frac{v_i}{\|v_i\|_2},
\]
with the sign and scale absorbed later into the magnitude \(\lambda_i\).

\subsection{Step 2: Solving for the Input-space Direction \texorpdfstring{\(k_i\)}{k(i)}}
\label{sec: solving input space direction}

Having solved the output-space direction \(u_i = \hat{v}_i\), we now determine
the input-space direction \(k_i\) for each editable component \(W_i\).
As before, we solve for the \emph{direction} of \(k_i\); its sign and scale
are absorbed into the scalar magnitude \(\lambda_i\).

\paragraph{Intuition.}
As shown in Eq.~\eqref{e:delta o}, the input-space direction $k_i$ determines
which input activations $h_i$ trigger the edit through the inner product
$k_i^\top h_i$.
To identify a suitable \(k_i\), we note that a
well-trained component \(W_i\) already encodes which inputs are relevant for
contributing to the semantic direction \(v_i\), and the edit \(\Delta W_i\) should mirror this
existing input-dependent pattern.

To formalize this intuition, we define the \emph{semantic alignment score} of $W_i$
for an input $h_i$ as
\[
s_i(h_i) := v_i^\top o_i = v_i^\top W_i h_i,
\]
which measures how strongly the original component output
aligns with the semantic direction $v_i$ for a given input $h_i$.
Intuitively, if \(s_i(h_i)\) is small across inputs, this indicates that the component is generally unrelated to the \(v_i\), and the edit should be small for all inputs similarly.
If, for some components, \(s_i(h_i)\) is large for certain inputs, the edit should
be large on those same inputs.

Hence, we choose $k_i$ so that the induced change in the semantic alignment score,
\(
\Delta s_i(h_i) := v_i^\top \Delta W_i h_i,
\)
occurs on the same inputs for which $s_i(h_i)$ is large. To formalize this idea, we maximize the
"absolute" Pearson correlation between \(\Delta s_i(h_i)\) and \(s_i(h_i)\), as we do not care about the sign or overall scale at this stage. The following theorem
provides the solution.

\begin{theorem}[Input-space direction matching semantic alignment variation]
\label{thm:ki}
Fix a component \(W_i\) and set \(u_i=\hat v_i\).
Assume \(W_i^{\top} v_i \neq 0\) and
\(\operatorname{Var}(s_i(h_i)) > 0\), where
\(s_i(h_i) := v_i^{\top} W_i h_i\).
Consider choosing an input-direction \(k_i \neq 0\) so that the induced semantic alignment
shift \(\Delta s_i(h_i) := v_i^{\top} \Delta W_i h_i\) exhibits maximal co-variation with the
component’s intrinsic semantic alignment score \(s_i(h_i)\).
Formally, consider the objective:
\[
\max_{k_i \neq 0}
\Big|\operatorname{Pearson}\!\big(
\Delta s_i(h_i),\;
s_i(h_i)
\big)\Big|.
\]
Then there exists a maximizer \(k_i\) that is collinear with \(W_i^{\top} v_i\), i.e.,
\[
k_i \in \operatorname{span}\{W_i^{\top} v_i\}.
\]
\end{theorem}

We defer the proof to Appendix~\ref{app:proof-ki}. Theorem~\ref{thm:ki} shows that the input-space direction $k_i$ should align with the
component’s intrinsic input sensitivity $W_i^\top v_i$.
We therefore adopt the normalized choice
\[
\hat{k}_i := \frac{W_i^{\top} v_i}{\|W_i^{\top} v_i\|_2}.
\]
This choice is further empirically
validated in Appendix~\ref{app:ablation}.

\subsection{Step 3: Solving for the Edit Magnitudes \texorpdfstring{\(\boldsymbol{\lambda}\)}{lambda}}
\label{sec: solving edit magnitudes}

With the edit directions $u_i$ and $k_i$ fixed, we now determine the magnitudes
$\{\lambda_i\}$, which control how strongly each component is reinforced or
suppressed.
Intuitively, the magnitude assigned to each component should reflect how that
component contributes to the direction $v_i$ \emph{on average across inputs}:
components that consistently align with $v_i$ should be reinforced, components
that consistently oppose it should be suppressed, and components with weak alignment should receive little or no edit.

Note that this role is fundamentally different from that of the input-space direction
$k_i$, which captures how the component’s semantic alignment score $s_i(h_i)$ \emph{varies across
inputs}, whereas the magnitudes $\lambda_i$ depend only on the component’s \emph{averaged} semantic alignment.

To formalize this allocation of editing strength, we now introduce an importance
weighting for each component and derive \(\lambda_i\) via a global regularized
optimization.

\paragraph{Importance weighting.}
Recall that the \emph{semantic alignment score}
$s_i(h_i) := v_i^{\top} W_i h_i$
measures how strongly component $W_i$ contributes to the semantic direction $v_i$
for a given input $h_i$.
Since the magnitudes $\lambda_i$ are intended to capture a component’s \emph{overall} semantic contribution, we measure this contribution by the expectation of the semantic alignment score over the input distribution: 
\(\mathbb{E}[s_i(h_i)] = \mathbb{E}[v_i^{\top} W_i h_i]\). However, the typical output magnitude of \(W_i h_i\) can vary substantially across
transformer layers, making raw values not directly comparable between components.
To place components on a common footing, we make the following changes: (i) remove the arbitrary scale of the
semantic direction by using \(\hat v_i = v_i/\|v_i\|_2\), and (ii) normalize by the
output norm of the mean activation \(\mu_i=\mathbb{E}[h_i]\), yielding
\vspace{-5pt}
\[
g_i
=
\frac{\mathbb{E}[s_i(h_i)]}{\|v_i\|_2\,\|W_i \mu_i\|_2}
=
\frac{v_i^\top W_i \mu_i}{\|v_i\|_2\,\|W_i \mu_i\|_2}
=
\cos(v_i, W_i \mu_i),
\]
which we refer to as the \emph{component importance score}.
The sign of \(g_i\) indicates whether the component aligns or opposes the semantic
direction, while \(|g_i|\) measures the strength of this tendency.
This normalization choice is empirically validated in
Appendix~\ref{app:ablation}.

\paragraph{Elastic-Net objective.}
The component importance score \(g_i\) indicates whether component \(W_i\) should be
reinforced or suppressed, and with what strength. A natural objective is therefore to maximize total alignment
\(
\boldsymbol{g}^{\top}\boldsymbol{\lambda}
=
\sum_{i=1}^n g_i \lambda_i.
\)
However, this objective is unbounded, and effective weight editing should remain lightweight by modifying only a small number of relevant components while keeping edit magnitudes controlled.

We address both considerations with an Elastic-Net regularization, combining an
\(\ell_1\) term to promote sparsity and an \(\ell_2\) term to limit overall edit
size:
\vspace{-5pt}
\[
\max_{\boldsymbol{\lambda}}
\;
\boldsymbol{g}^{\top} \boldsymbol{\lambda}
-
\rho\!\left(
\alpha\|\boldsymbol{\lambda}\|_1
+
\frac{1-\alpha}{2}\|\boldsymbol{\lambda}\|_2^2
\right),
\]
where $\rho > 0$ controls the global edit budget and $\alpha \in [0,1)$ trades off
$\ell_1$ sparsity and $\ell_2$ smoothness.
Ablation results in Appendix~\ref{app:ablation} confirm the importance of both the $\ell_1$ and $\ell_2$ regularization terms.

\begin{theorem}[Edit magnitude allocation under regularization]
\label{thm:lambda}
For each component \(W_i\), let
\(
g_i = \cos(v_i, W_i\mu_i)
\)
denote the component importance score, with the convention that \(g_i := 0\) if
\(W_i\mu_i = 0\). Consider the problem of assigning edit magnitudes $\{\lambda_i\}$
to maximize total signed alignment as measured by the component importance scores \(\{g_i\}\),
while controlling both the sparsity and overall strength of the edit.
Formally, let \(\boldsymbol{g}=(g_1,\dots,g_n)\) and consider
\[
\max_{\boldsymbol{\lambda}\in\mathbb{R}^n}
\;
\boldsymbol{g}^{\top}\boldsymbol{\lambda}
-
\rho\!\left(
\alpha\|\boldsymbol{\lambda}\|_1
+
\frac{1-\alpha}{2}\|\boldsymbol{\lambda}\|_2^2
\right),
 \rho>0,\ \alpha\in[0,1).
\]
The unique edit magnitude assigned
to component \(i\) is
\[
\lambda_i^*
=
\operatorname{sign}(g_i)\,
\frac{\max(|g_i|-\rho\alpha,\,0)}{\rho(1-\alpha)}.
\]
\end{theorem}

We defer the proof to Appendix~\ref{app:proof-lambda}. Theorem~\ref{thm:lambda} yields a closed-form soft-thresholding rule for allocating
edit magnitudes \(\lambda_i\) according to the alignment scores \(g_i\) under a global
Elastic-Net budget. Substituting \(\lambda_i^*\), together with
\(u_i=\hat v_i\) and \(k_i=\hat k_i\), gives the unified weight-editing update below.

\subsection{Summary: Unified Weight Editing Rule}

Each editable component \(W_i\) receives the rank-1 update
\[
\boxed{
\Delta W_i
=
\operatorname{sign}(g_i)\,
\frac{\max(|g_i| - \rho\alpha,\,0)}{\rho(1-\alpha)}
\;\hat{v}_i\, \hat{k}_i^{\top},
}
\]
where
\(
\hat{v}_i := \frac{v_i}{\|v_i\|_2}, \;\;\;
\hat{k}_i = \frac{W_i^{\top} v_i}{\|W_i^{\top} v_i\|_2}, \;\;\;
g_i = \cos(v_i,\, W_i \mu_i).
\) Each update is:
\begin{itemize}[leftmargin=1.2em]
\item \textbf{Directionally selective:} it modifies only the projection along the semantic direction \(v_i\);
\item \textbf{Input-selective:} it determines which inputs should trigger the edit based on the component’s intrinsic semantic behavior;
\item \textbf{Budget-aware:} its magnitude \(\lambda_i^*\) is determined by the component importance score \(g_i\) under
an Elastic-Net regularizer.
\end{itemize}
\vspace{-5pt}

Thus, \textsc{\textbf{Steer2Edit}} yields a component-level weight editing framework that
jointly captures \emph{what} semantic direction to modify, \emph{when} the edit should be
activated by the input, and \emph{how strongly} each component should be adjusted.

%% file: sec/Experiments.tex
\section{Experiments}
\label{sec:experiments}

This section evaluates whether \textsc{\textbf{Steer2Edit}} achieves a superior
\emph{attribute--utility trade-off} compared to inference-time activation steering.
We consider three representative behavioral control settings:
(i) safety alignment against jailbreak attacks,
(ii) truthfulness promotion,
and (iii) reasoning efficiency control.
Across all settings, we compare \textsc{\textbf{Steer2Edit}} against three steering baselines:


\vspace{-5pt}
\paragraph{Steering Baselines.}
We compare against three activation steering methods. \textbf{RepEngineering} \citep{representation_engineering} extracts block-level steering vectors $\{v^{\text{attn}}_\ell, v^{\text{mlp}}_\ell\}_{\ell=1}^L$ at every layer and shifts each block output as $h_\ell \leftarrow h_\ell + \gamma v_\ell$ during inference; we sweep $\gamma$ to trace trade-off curves. \textbf{CAA} \citep{panickssery2023steering} adds a single residual-stream direction at layer $\lfloor L/3 \rfloor$. \textbf{ODESteer} \citep{zhao2026odesteer} applies adaptive ODE-based steering via a learned barrier function at the same layer. The construction of steering vectors for each baseline is deferred to
Appendix~\ref{app:steering_vectors}.
\vspace{-5pt}
\paragraph{Steer2Edit.}
Following Section~\ref{sec:steer2edit}, we apply rank-1 edits to linear components that
write directly into the residual stream.
For attention, we edit each head’s output projection
$W_o \in \mathbb{R}^{d_{\text{model}} \times d_{\text{head}}}$.
For MLPs, we edit individual down-projection neurons by treating each column
$w_{\text{down},j}$ of $W_{\text{down}} \in \mathbb{R}^{d_{\text{model}} \times d_{\text{ff}}}$
as an independent component. Edit magnitudes are determined by an Elastic-Net objective with sparsity parameter
$\alpha$ and global budgets $\rho_{\text{attn}}$ and $\rho_{\text{mlp}}$ for attention
heads and MLP neurons.
For each setting, hyperparameters
$(\rho_{\text{attn}}, \rho_{\text{mlp}}, \alpha)$
are selected via a small grid search on a held-out validation set,
ranking configurations by the target attribute metric.
Unless otherwise noted, results correspond to the best-performing
configurations.
Full search details are provided in Appendix~\ref{app:hyperparam}.

\subsection{Evaluation for Behavioral Control}
\label{subsec:usecases}

We consider three settings that examine behavioral control along distinct
dimensions: safety alignment, truthfulness, and reasoning efficiency.
In each setting, the \emph{target attribute} is measured using task-specific metrics,
while \emph{downstream utility} is evaluated on task-oriented benchmarks that
are unrelated to the controlled behavior. We visualize the trade-off between the target attribute and utility.

\subsubsection{Safety Alignment Against Jailbreak Attacks}
\label{subsec:safety}

\begin{figure}[!t]
    \centering
\vspace{-10pt}    \includegraphics[width=0.9\linewidth]{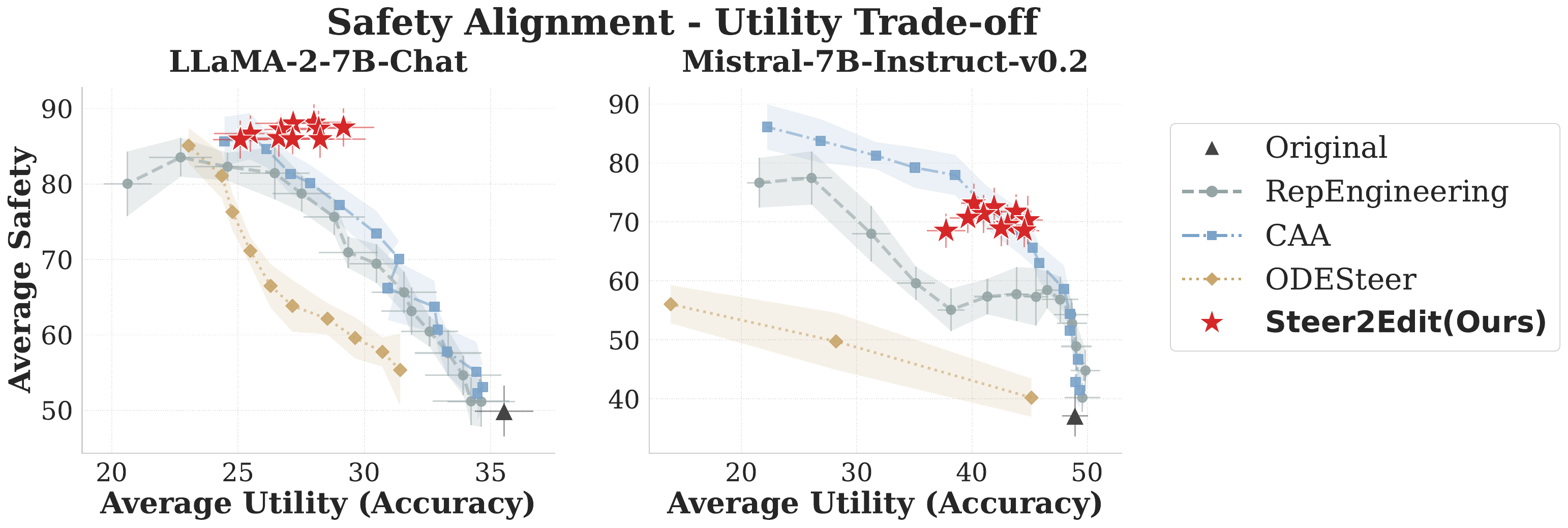}
    \vspace{-5pt}
    \caption{Safety--utility trade-off on \textbf{LLaMA-2-7B-Chat} and
    \textbf{Mistral-7B-Instruct-v0.2}. Each curve traces one steering
    baseline;
    each star is one \textsc{\textbf{Steer2Edit}} configuration.
    \textsc{\textbf{Steer2Edit}} consistently occupies the upper-right region,
    achieving higher refusal rates at comparable or higher utility.}
    \vspace{-5pt}
    \label{fig:safety_tradeoff}
\end{figure}

\begin{figure}[!b]
    \centering
    \includegraphics[width=0.9\linewidth]{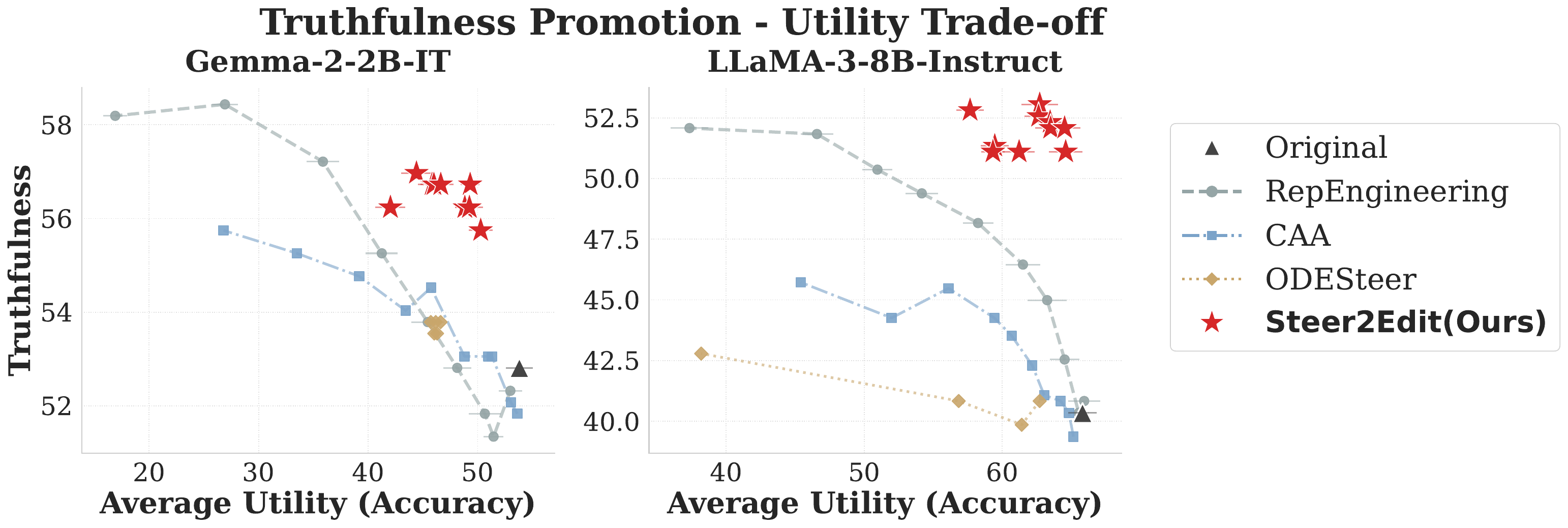}
    \vspace{-5pt}
    \caption{Truthfulness--utility trade-off on \textbf{Gemma-2-2B-IT} and
    \textbf{LLaMA-3-8B-Instruct}. \textsc{\textbf{Steer2Edit}} improves truthfulness while maintaining higher utility than all baselines.}
    \vspace{-10pt}
    \label{fig:truthfulness_tradeoff}
\end{figure}

\paragraph{Goal.}
We evaluate whether \textsc{\textbf{Steer2Edit}} can strengthen refusal behavior under strong
jailbreak attacks while preserving helpfulness on benign tasks.
\vspace{-5pt}
\paragraph{Models and evaluation.}
We evaluate safety alignment on \textbf{LLaMA-2-7B-Chat} and
\textbf{Mistral-7B-Instruct-v0.2}.
Safety is measured using \textbf{ADVBench}, which consists of harmful user queries
designed to elicit unsafe behavior.
Each query is transformed into a jailbreak prompt using either
\textbf{GCG} \citep{gcg}, a classical gradient-based attack, or \textbf{ADV-LLM} \citep{advllm}, a substantially
stronger attack that trains an LLM to generate adversarial suffixes.
We report the \textbf{refusal rate}, defined as the proportion of model responses that are refusals, averaged across both attack types. Downstream utility is evaluated on \textbf{GSM8K}, \textbf{CodeMMLU}, and
\textbf{CommonsenseQA}, measuring utility on grade-school math reasoning, programming, and commonsense multiple-choice questions. Utility is reported as the mean accuracy across all three benchmarks.
\vspace{-5pt}
\paragraph{Results.}
Figure~\ref{fig:safety_tradeoff} illustrates the safety--utility relationship. Overall, \textsc{\textbf{Steer2Edit}} lies beyond the trade-off frontier of all baselines, occupying the upper-right region of the plot. Notably, only CAA on Mistral-7B reaches a trade-off region comparable to \textsc{\textbf{Steer2Edit}}; however, as we will later show in Sections~\ref{subsec:truthfulness} and~\ref{subsec:reasoning}, this competitiveness does not generalize to the truthfulness or reasoning efficiency settings, where CAA exhibits much weaker performance.

\begin{figure}[!t]
    \centering
    \includegraphics[width=0.9\linewidth]{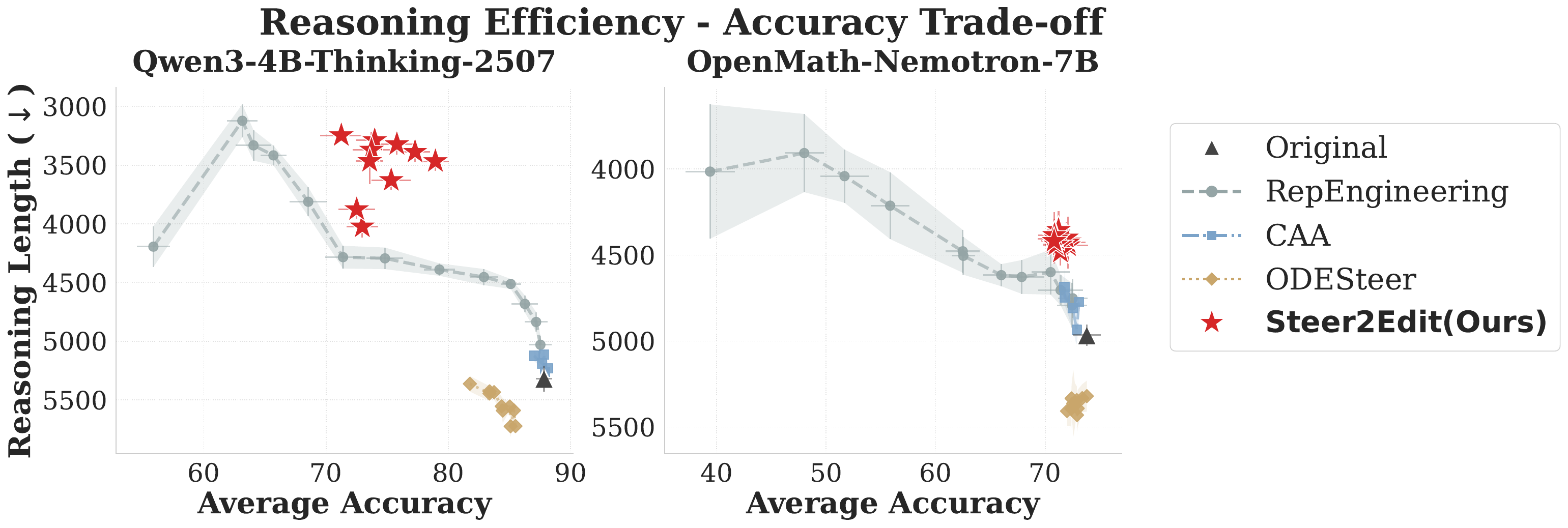}
    \vspace{-5pt}
    \caption{Accuracy--efficiency trade-off on \textbf{Qwen3-4B-Thinking-2507} and
    \textbf{OpenMath-Nemotron-7B}. \textsc{\textbf{Steer2Edit}} achieves shorter reasoning traces at comparable accuracy}
    \vspace{-10pt}
    \label{fig:reasoning_tradeoff}
\end{figure}

\subsubsection{Truthfulness Promotion}
\label{subsec:truthfulness}

\paragraph{Goal.}
We evaluate whether \textsc{\textbf{Steer2Edit}} increases the model’s preference for
truthful answers while preserving
performance on unrelated downstream tasks.

\vspace{-5pt}
\paragraph{Models and evaluation.}
We evaluate on \textbf{Gemma-2-2B-IT} and \textbf{LLaMA-3-8B-Instruct} using
\textbf{TruthfulQA}.
For each prompt, we measure whether the model assigns higher probability to the truthful
answer than to a plausible but false alternative, and report truthful preference accuracy.
Downstream utility is again measured on GSM8K, CodeMMLU, and CommonsenseQA.

\vspace{-5pt}
\paragraph{Results.}
Figure~\ref{fig:truthfulness_tradeoff} shows the truthfulness--utility relationship.
\textsc{\textbf{Steer2Edit}} achieves substantially higher truthfulness at comparable utility, dominating the Pareto frontier of all three baselines.

\subsubsection{Efficient Reasoning}
\label{subsec:reasoning}

\paragraph{Goal.}
We evaluate whether \textsc{\textbf{Steer2Edit}} can shorten reasoning traces while preserving
answer accuracy, improving inference efficiency for Large Reasoning Models (LRMs).
\vspace{-5pt}
\paragraph{Models and datasets.}
We evaluate on \textbf{Qwen3-4B-Thinking-2507} and \textbf{OpenMath-Nemotron-7B} using
GSM8K, MATH-500, GPQA, and CodeMMLU.
Downstream utility is measured as mean accuracy across all datasets, while reasoning
efficiency is measured by the number of generated reasoning tokens.

\vspace{-5pt}
\paragraph{Results.}
Figure~\ref{fig:reasoning_tradeoff} shows the accuracy--efficiency relationship. CAA and ODESteer produce almost no length reduction, consistent with the component analysis in the next section~\ref{subsec:interpretability} showing that reasoning efficiency is governed by distributed MLP neurons across many layers, limiting single-layer methods. \textsc{\textbf{Steer2Edit}} significantly shortens reasoning traces while maintaining comparable accuracy by distributing edits across the full network.

\subsection{Component-Level Interpretability}
\label{subsec:interpretability}
Because \textsc{\textbf{Steer2Edit}} applies edits at the level of individual attention heads and MLP neurons, the resulting weight updates are directly interpretable and reveal which components mediate each target behavior. We visualize the signed edit coefficients $\lambda$ for the best-performing configuration in each setting (Appendix~\ref{app:edit_heatmaps}). A clear structural difference emerges: \textbf{safety alignment and truthfulness} are predominantly mediated by sparse edits to a small number of attention heads, while \textbf{reasoning efficiency} is governed by dense, distributed edits to MLP neurons across many layers. This structural divergence confirms that different behavioral properties are encoded in distinct component types, and that \textsc{Steer2Edit}'s component-level decomposition captures this structure automatically.

\subsection{Additional Experiments}
For completeness, we report the following additional experiments in the Appendix:
\begin{itemize}[leftmargin=*]
    \item \textbf{Per-dataset trade-offs} (Appendix~\ref{sec:appendix_tradeoff}): Trade-off curves for each evaluation dataset.
    \item \textbf{Design-choice ablations} (Appendix~\ref{app:ablation}): Ablations of the input direction, importance score, and edit magnitude.
    \item \textbf{Component-wise budget sensitivity} (Appendix~\ref{app:component_budget}): Sensitivity to attention and MLP editing budgets.
    \item \textbf{Training-based baselines} (Appendix~\ref{app:finetune_baselines}): Comparisons with SFT and LoRA.
    \item \textbf{Weight-editing baseline} (Appendix~\ref{app:alphaedit}): Comparison with AlphaEdit~\citep{fang2024alphaedit}.
    \item \textbf{Reasoning-compression baseline} (Appendix~\ref{app:asc}): Comparison with ASC~\citep{azizi2025activation}.
    \item \textbf{Over-refusal evaluation} (Appendix~\ref{app:orbench}): Evaluation on OR-Bench~\citep{cui2024or}.
    \item \textbf{Scaling to 70B models} (Appendix~\ref{app:scaling_70b}): Evaluation on larger language models.
    \item \textbf{Vision-language models} (Appendix~\ref{app:vlm}): Generalization to VLMs.
    \item \textbf{Diffusion language models} (Appendix~\ref{app:dream}): Generalization to diffusion LLMs.
    \item \textbf{Fine-grained steering baselines} (Appendix~\ref{app:fine_grained_steering}): Comparisons with ITI and AUSteer.
    \item \textbf{Labeling-judge robustness} (Appendix~\ref{app:judge_robustness}): Sensitivity to the truthfulness labeling judge.
\end{itemize}

%% file: sec/RelatedWorks.tex
\section{Related Works}

\paragraph{Steering and controlling LLM behavior.}
A growing body of work studies behavioral control in LLMs via interventions on internal representations.
Activation steering methods \citep{representation_engineering, steering, refusal, reflctrl, unlearning, panickssery2023steering, zhao2026odesteer, sun2026llm}
extract semantic directions from contrastive examples and apply inference-time activation interventions to modulate attributes such as safety or reasoning behavior. These approaches are training-free, but rely on global, inference-time modifications. In parallel, Concept Bottleneck Models \citep{cbllm, cbllm_classification, conceptlayers, c3m, tbm}
introduce explicit architectural constraints to enable human-interpretable control.
Together, these lines of work demonstrate that LLM behavior can be influenced through manipulation of internal representations.

In contrast, \textsc{\textbf{Steer2Edit}} translates steering signals into
\emph{component-level} edits that operate at the level of individual components.
By redistributing behavioral influence across attention heads and MLP neurons,
\textsc{\textbf{Steer2Edit}} enables more favorable attribute--utility trade-offs,
while providing fine-grained interpretability and preserving the standard model architecture.

\vspace{-5pt}
\paragraph{LLM weight editing.}
Another line of work focuses on modifying model parameters to induce persistent behavioral changes without full retraining.
Representative approaches include meta-editors such as MEND \citep{mend} and KnowledgeEditor \citep{knowledgeeditor},
mechanistic editors such as ROME \citep{rome} and MEMIT \citep{memit},
neuron-level interventions \citep{knowledgeneurons},
semi-parametric methods such as SERAC \citep{serac},
and null-space constrained editing such as AlphaEdit \citep{fang2024alphaedit}.
More recent work applies targeted weight edits to specific components for behavior control,
including ThinkEdit \citep{thinkedit} for improving reasoning quality and CBSE \citep{drefa,chu2026detection} for safety. However, unlike \textsc{\textbf{Steer2Edit}}, these methods are largely empirical and do not provide a unified framework for allocating and justifying edits across components.

%% file: sec/Conclusion.tex
\section{Conclusion}
We introduced \textsc{\textbf{Steer2Edit}}, a theoretically grounded framework that translates steering signals into component-level weight edits via a closed-form solution. By shifting behavioral control from inference-time activation intervention to parameter updates, \textsc{\textbf{Steer2Edit}} achieves more favorable behavior--utility trade-offs while preserving the standard model architecture. Beyond empirical gains, the method offers fine-grained interpretability, revealing how safety, truthfulness, and reasoning efficiency are distributed across attention heads and MLP components.

%% file: sec/Appendix.tex
\newpage

\section{Broader Impact}
\label{app:broader_impact}

This paper presents work whose goal is to advance the field of machine learning by enabling efficient, interpretable, and training-free model editing at the level of individual components. In positive applications, \textsc{Steer2Edit} may help practitioners correct or reduce undesirable behaviors (e.g., unsafe responses or hallucinations) and better understand which internal components support a given behavior, which can improve transparency and facilitate auditing. At the same time, weight-editing methods are inherently dual-use: the same capability can be applied to remove safeguards, amplify biases, or otherwise manipulate model behavior for harmful purposes, and edited models may be redistributed without clear provenance. To mitigate these risks, we emphasize that edits should be evaluated across diverse safety and capability tests, that releases should include clear documentation of intended use and limitations, and that responsible access controls may be appropriate for edits that materially alter safety-critical behaviors. Overall, we believe the primary societal consequence of this work is enabling more controllable and inspectable models, with corresponding responsibility to prevent and detect misuse.

\section{Proofs for \textsc{\textbf{Steer2Edit}}}
\label{app:steer2edit_proofs}

\subsection{Proof of Theorem~\ref{thm:ui}}
\label{app:proof-ui}
\begin{theorem}[Output-space direction under semantic invariance]
\label{thm:ui_app}
Let \(v_i \neq 0\), and let \(\Delta W_i = \lambda_i\, u_i k_i^{\top}\) be a
rank-1 edit with \(\Delta W_i \neq 0\).
If for all \(h_i\) and all \(z \perp v_i\) we have
\[
z^{\top} \Delta W_i h_i = 0,
\]
then the output-space direction \(u_i\) must be collinear with \(v_i\), i.e.,
\[
u_i \in \operatorname{span}\{v_i\}.
\]
\end{theorem}

\begin{proof}
Substituting \(\Delta W_i = \lambda_i\, u_i k_i^{\top}\),
\[
z^{\top} \Delta W_i h_i
=
\lambda_i\,(z^{\top} u_i)\,(k_i^{\top} h_i).
\]
Because \(\Delta W_i \neq 0\), we have \(k_i \neq 0\), and therefore there exists
some \(h_i\) such that \(k_i^{\top} h_i \neq 0\).
For the expression above to vanish for all such \(h_i\), we must have
\(z^{\top} u_i = 0\) for every vector \(z\) orthogonal to \(v_i\).
The only vectors satisfying this condition are those proportional to \(v_i\).
Hence \(u_i \in \operatorname{span}\{v_i\}\).
\end{proof}

\subsection{Proof of Theorem~\ref{thm:ki}}
\label{app:proof-ki}
\begin{theorem}[Input-space direction matching semantic alignment variation]
\label{thm:ki_app}
Fix a component \(W_i\) and set \(u_i=\hat v_i\).
Assume \(W_i^{\top} v_i \neq 0\) and
\(\operatorname{Var}(s_i(h_i)) > 0\), where
\(s_i(h_i) := v_i^{\top} W_i h_i\).
Consider choosing an input-direction \(k_i \neq 0\) so that the induced semantic alignment
shift \(\Delta s_i(h_i) := v_i^{\top} \Delta W_i h_i\) exhibits maximal co-variation with the
component’s intrinsic semantic alignment score \(s_i(h_i)\).
Formally, consider the objective
\[
\max_{k_i \neq 0}
\Big|\operatorname{Pearson}\!\big(
\Delta s_i(h_i),\;
s_i(h_i)
\big)\Big|.
\]
Then there exists a maximizer \(k_i\) that is collinear with \(W_i^{\top} v_i\), i.e.,
\[
k_i \in \operatorname{span}\{W_i^{\top} v_i\}.
\]
\end{theorem}

\begin{proof}
Recall that
\(s_i(h_i) := v_i^{\top} W_i h_i\) and
\(\Delta s_i(h_i) := v_i^{\top} \Delta W_i h_i.
\)
Using \(\Delta W_i = \lambda_i \hat v_i k_i^{\top}\), we have
\[
\Delta s_i(h_i)
= \lambda_i \|v_i\|_2 \,(k_i^{\top} h_i).
\]

Pearson correlation is invariant to additive shifts in either argument, so it is
unchanged if we center the inputs.
Let \(\tilde{h}_i = h_i - \mu_i\) with \(\mu_i = \mathbb{E}[h_i]\), and define
\[
\widetilde{\Delta s}_i(h_i) := \Delta s_i(\tilde{h}_i),
\qquad
\tilde{s}_i(h_i) := s_i(\tilde{h}_i).
\]
Denote the covariance matrix by
\[
\Sigma_i := \mathbb{E}[\tilde{h}_i \tilde{h}_i^{\top}].
\]

\paragraph{Step 1: covariance.}
We have
\[
\widetilde{\Delta s}_i(h_i)
= \lambda_i\|v_i\|_2 \,(k_i^{\top}\tilde{h}_i),
\qquad
\tilde{s}_i(h_i)
= v_i^{\top}W_i \tilde{h}_i.
\]
Hence
\begin{align*}
\operatorname{Cov}(\widetilde{\Delta s}_i,\tilde{s}_i)
&= \mathbb{E}[\widetilde{\Delta s}_i(h_i)\,\tilde{s}_i(h_i)] \\
&= \lambda_i\|v_i\|_2 \,\mathbb{E}[(k_i^{\top}\tilde{h}_i)(v_i^{\top}W_i \tilde{h}_i)] \\
&= \lambda_i\|v_i\|_2 \,k_i^{\top}\,\mathbb{E}[\tilde{h}_i \tilde{h}_i^{\top}]\,W_i^{\top} v_i \\
&= \lambda_i\|v_i\|_2 \,k_i^{\top}\Sigma_i W_i^{\top} v_i.
\end{align*}

\paragraph{Step 2: variances.}
Similarly,
\[
\operatorname{Var}(\widetilde{\Delta s}_i)
= \mathbb{E}[\widetilde{\Delta s}_i(h_i)^2]
= \lambda_i^2\|v_i\|_2^2\,k_i^{\top}\Sigma_i k_i,
\]
and
\[
\operatorname{Var}(\tilde{s}_i)
= \mathbb{E}[\tilde{s}_i(h_i)^2]
= v_i^{\top}W_i\Sigma_i W_i^{\top}v_i.
\]

\paragraph{Step 3: Pearson correlation.}
The Pearson correlation between the induced and intrinsic semantic signals is
\[
\operatorname{Pearson}(\widetilde{\Delta s}_i,\tilde{s}_i)
=
\frac{\operatorname{Cov}(\widetilde{\Delta s}_i,\tilde{s}_i)}
     {\sqrt{\operatorname{Var}(\widetilde{\Delta s}_i)}\,
      \sqrt{\operatorname{Var}(\tilde{s}_i)}}.
\]
Substituting the expressions above gives
\begin{align*}
\operatorname{Pearson}(\widetilde{\Delta s}_i,\tilde{s}_i)
&=
\frac{\lambda_i\|v_i\|_2 \,k_i^{\top}\Sigma_i W_i^{\top} v_i}
     {\sqrt{\lambda_i^2\|v_i\|_2^2\,k_i^{\top}\Sigma_i k_i}\,
      \sqrt{v_i^{\top}W_i\Sigma_i W_i^{\top}v_i}} \\[4pt]
&=
\frac{\operatorname{sign}(\lambda_i)\,k_i^{\top}\Sigma_i W_i^{\top} v_i}
     {\sqrt{k_i^{\top}\Sigma_i k_i}\,
      \sqrt{v_i^{\top}W_i\Sigma_i W_i^{\top}v_i}}.
\end{align*}
The denominator’s second factor is independent of \(k_i\), and
\(\operatorname{sign}(\lambda_i)\) is irrelevant when maximizing absolute
correlation.
Thus maximizing
\(|\operatorname{Pearson}(\widetilde{\Delta s}_i,\tilde{s}_i)|\)
over \(k_i\neq 0\) reduces to maximizing
\[
\frac{|k_i^{\top}\Sigma_i W_i^{\top} v_i|}
     {\sqrt{k_i^{\top}\Sigma_i k_i}}.
\tag{1}
\]

\paragraph{Step 4: Cauchy--Schwarz in the \texorpdfstring{\(\Sigma_i\)}{Sigma(i)}-inner product.}
Define the \(\Sigma_i\)-inner product:
\[
\langle a,b\rangle_{\Sigma_i} := a^{\top}\Sigma_i b,
\qquad
\|a\|_{\Sigma_i} := \sqrt{a^{\top}\Sigma_i a}.
\]
Then (1) becomes
\[
\frac{|\langle k_i, W_i^{\top} v_i\rangle_{\Sigma_i}|}{\|k_i\|_{\Sigma_i}}.
\]
By Cauchy--Schwarz,
\[
|\langle k_i, W_i^{\top} v_i\rangle_{\Sigma_i}|
\le
\|k_i\|_{\Sigma_i}\,\|W_i^{\top} v_i\|_{\Sigma_i}.
\]
Equality is attained by choosing
\(k_i \propto W_i^{\top} v_i\).
Therefore, there exists an optimizer
\(k_i \in \operatorname{span}\{W_i^{\top} v_i\}\).
\end{proof}

\subsection{Proof of Theorem~\ref{thm:lambda}}
\label{app:proof-lambda}
\begin{theorem}[Edit magnitude allocation under regularization]
\label{thm:lambda_app}
For each component \(W_i\), let
\(
g_i = \cos(v_i, W_i\mu_i)
\)
denote the component importance score, with the convention that \(g_i := 0\) if
\(W_i\mu_i = 0\).

Consider the problem of assigning edit magnitudes $\{\lambda_i\}$
to maximize total signed alignment as measured by the component importance scores \(\{g_i\}\),
while controlling both the sparsity and overall strength of the edit.
Formally, let \(\boldsymbol{g}=(g_1,\dots,g_n)\) and consider
\[
\max_{\boldsymbol{\lambda}\in\mathbb{R}^n}
\;
\boldsymbol{g}^{\top}\boldsymbol{\lambda}
-
\rho\!\left(
\alpha\|\boldsymbol{\lambda}\|_1
+
\frac{1-\alpha}{2}\|\boldsymbol{\lambda}\|_2^2
\right),
 \rho>0,\ \alpha\in[0,1).
\]
The unique edit magnitude assigned
to component \(i\) is
\[
\lambda_i^*
=
\operatorname{sign}(g_i)\,
\frac{\max(|g_i|-\rho\alpha,\,0)}{\rho(1-\alpha)}.
\]
\end{theorem}

\begin{proof}
For component \(i\), define the one-dimensional objective
\[
J(\lambda_i)
=
g_i \lambda_i
-
\rho\!\left(
\alpha |\lambda_i|
+
\frac{1-\alpha}{2}\lambda_i^2
\right).
\]
A scalar value \(\lambda_i^*\) maximizes \(J\) iff
\[
0 \in \partial J(\lambda_i^*),
\]
where the subgradient is needed only at \(\lambda_i = 0\) due to the
nondifferentiability of \(|\lambda_i|\). We analyze the three regions
\(\lambda_i>0\), \(\lambda_i<0\), and \(\lambda_i=0\).

\medskip
\noindent
\textbf{Case 1: \(\lambda_i > 0\).}
Here \(|\lambda_i|=\lambda_i\), so
\[
J(\lambda_i)
=
g_i \lambda_i
-
\rho\left(
\alpha\lambda_i
+
\frac{1-\alpha}{2}\lambda_i^2
\right).
\]
Differentiating gives
\[
\frac{dJ}{d\lambda_i}
=
g_i - \rho\alpha - \rho(1-\alpha)\lambda_i.
\]
Setting this to zero yields
\[
\lambda_i
=
\frac{g_i - \rho\alpha}{\rho(1-\alpha)},
\]
which is valid only when the positivity assumption holds, i.e.\ \(g_i > \rho\alpha\).

\medskip
\noindent
\textbf{Case 2: \(\lambda_i < 0\).}
Here \(|\lambda_i| = -\lambda_i\), so
\[
J(\lambda_i)
=
g_i \lambda_i
+
\rho\alpha \lambda_i
-
\frac{\rho(1-\alpha)}{2}\lambda_i^2.
\]
Differentiating,
\[
\frac{dJ}{d\lambda_i}
=
g_i + \rho\alpha - \rho(1-\alpha)\lambda_i.
\]
Setting this to zero gives
\[
\lambda_i
=
\frac{g_i + \rho\alpha}{\rho(1-\alpha)},
\]
which is valid only when \(g_i < -\rho\alpha\).

\medskip
\noindent
\textbf{Case 3: \(\lambda_i = 0\).}
At zero, the subdifferential of \(|\lambda_i|\) is \(\partial|\lambda_i|=[-1,1]\).  
The optimality condition
\[
0 \in g_i - \rho\alpha\, s, \qquad s\in[-1,1],
\]
is feasible iff \(|g_i|\le \rho\alpha\).  
Thus \(\lambda_i^* = 0\) whenever the alignment is too small to exceed the
threshold.

\medskip
Combining all three cases gives the soft-threshold rule
\[
\lambda_i^*
=
\begin{cases}
\displaystyle
\frac{g_i - \rho\alpha}{\rho(1-\alpha)}, & g_i > \rho\alpha,\\[6pt]
0, & |g_i| \le \rho\alpha,\\[6pt]
\displaystyle
\frac{g_i + \rho\alpha}{\rho(1-\alpha)}, & g_i < -\rho\alpha.
\end{cases}
\]
Equivalently,
\[
\lambda_i^*
=
\operatorname{sign}(g_i)\,
\frac{\max(|g_i| - \rho\alpha,\; 0)}{\rho(1-\alpha)}.
\]
\end{proof}

\newpage

\section{Component-Level Edit Visualizations}
\label{app:edit_heatmaps}

Figures~\ref{fig:safety_edit_distribution}--\ref{fig:reasoning_edit_distribution} visualize the signed edit coefficients $\lambda$ for the best-performing \textsc{Steer2Edit} configuration in each behavioral control setting. Each cell represents the magnitude of a rank-1 weight update applied to an attention head or MLP neuron: positive values (red) reinforce components aligned with the target behavior, while negative values (blue) suppress components that oppose it. These heatmaps reveal a clear structural divergence across tasks: safety and truthfulness are mediated by sparse, attention-dominated edits, while reasoning efficiency requires dense, distributed MLP modifications.

\begin{figure}[H]
    \centering
    \includegraphics[width=0.6\linewidth]{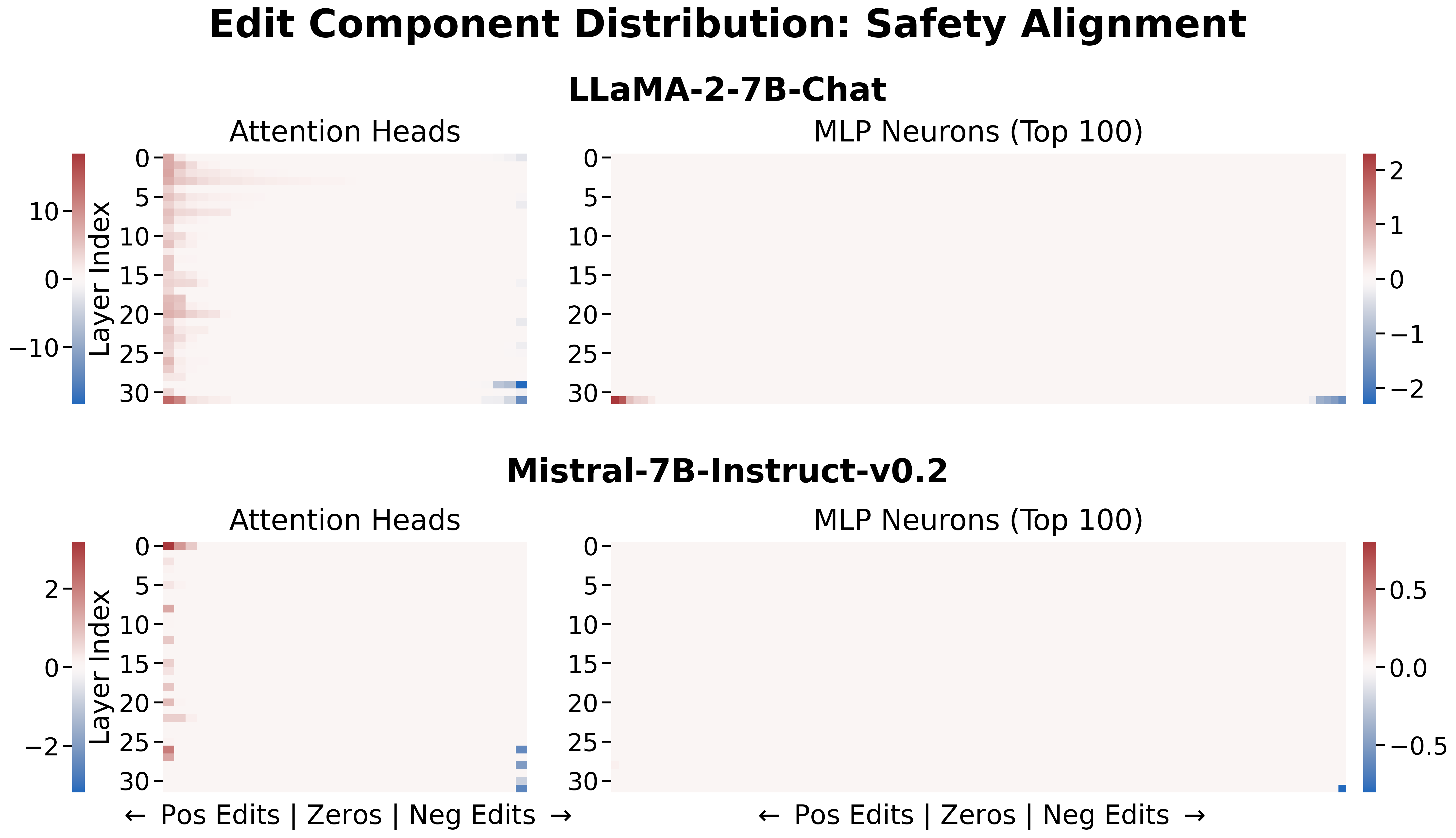}
    \vspace{-5pt}
    \caption{Signed \textsc{Steer2Edit} edit coefficients $\lambda$ for \textbf{safety alignment}. Edits are highly sparse and concentrated in a small subset of attention heads, predominantly in later layers. MLP neurons receive near-zero coefficients.}
    \label{fig:safety_edit_distribution}
\end{figure}
\vspace{-10pt}
\begin{figure}[H]
    \centering
    \includegraphics[width=0.6\linewidth]{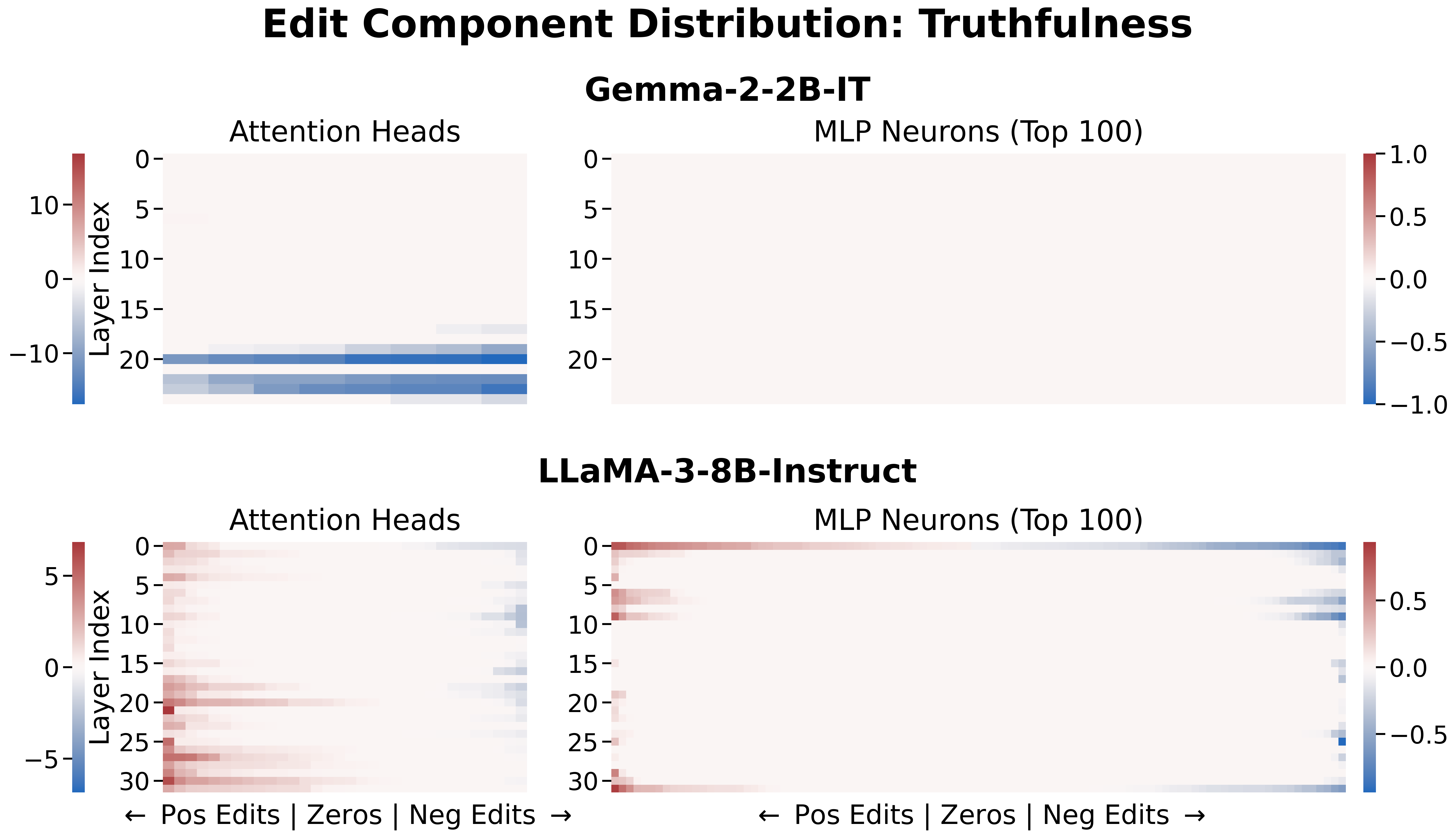}
    \vspace{-5pt}
    \caption{Signed \textsc{Steer2Edit} edit coefficients $\lambda$ for \textbf{truthfulness promotion}. Truthfulness control is sparse and predominantly mediated by attention heads. In \textbf{Gemma-2-2B-IT}, edits are dominated by negative coefficients, suggesting that truthfulness gains arise primarily from suppressing hallucination-promoting components.}
    \label{fig:truth_edit_distribution}
\end{figure}
\vspace{-10pt}
\begin{figure}[H]
    \centering
    \includegraphics[width=0.6\linewidth]{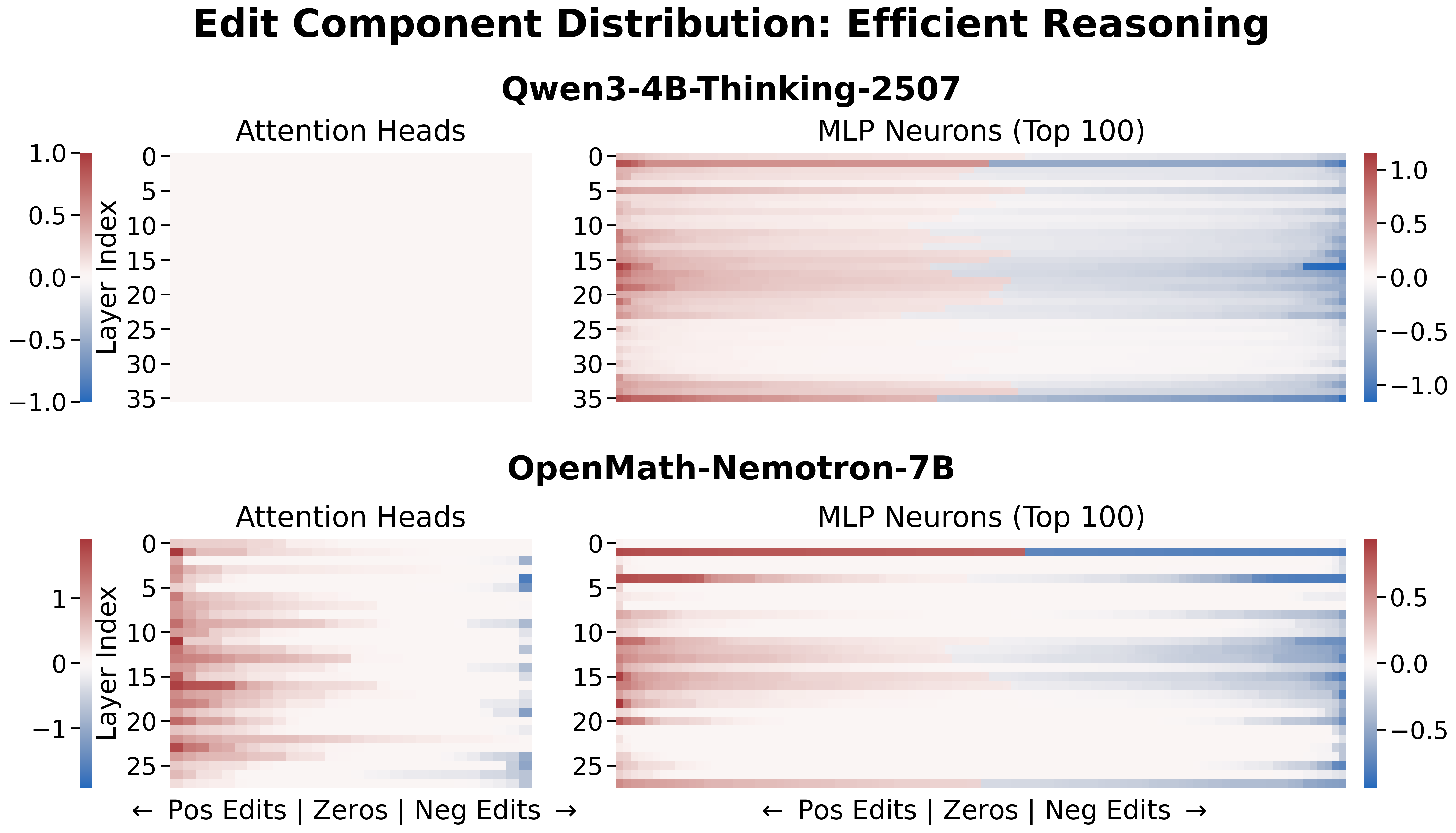}
    \vspace{-5pt}
    \caption{Signed \textsc{Steer2Edit} edit coefficients $\lambda$ for \textbf{reasoning efficiency}. In contrast to safety and truthfulness, reasoning efficiency is predominantly mediated by dense, distributed MLP edits spanning many neurons, while attention heads play a minor role.}
    \label{fig:reasoning_edit_distribution}
\end{figure}

\newpage

\section{Steering Vector Construction}
\label{app:steering_vectors}

This appendix describes the procedures used to construct steering vectors for each
behavioral control setting.
Across all experiments, steering vectors are computed using the same general
mean-difference formulation and are shared verbatim between activation steering and
\textsc{\textbf{Steer2Edit}}.
Only the definition of positive and negative response sets differs by task.

\paragraph{General formulation.}
For a given model, layer $\ell$, and block type (attention or MLP), we collect the
block output activations for a set of responses.
Let $\mathcal{P}$ and $\mathcal{N}$ denote the positive and negative response sets
associated with a target behavior.
For each response, we average the block outputs over all response tokens.
The steering vector is then defined as the difference between the mean activations:
\[
v_\ell = \mathbb{E}_{x \in \mathcal{P}}[h_\ell(x)] - \mathbb{E}_{x \in \mathcal{N}}[h_\ell(x)] .
\]
This procedure is applied independently to the attention and MLP blocks at each layer,
yielding $\{v^{\text{attn}}_\ell, v^{\text{mlp}}_\ell\}_{\ell=1}^L$.

\paragraph{Safety alignment.}
For safety alignment, we construct steering vectors using the ADVBench dataset.
The positive set $\mathcal{P}$ consists of refusal responses to harmful prompts,
while the negative set $\mathcal{N}$ consists of standard helpful responses to benign questions sampled from
Alpaca dataset. Steering vectors are computed from model-generated responses.

\paragraph{Truthfulness.}
For truthfulness promotion, we use the TruthfulQA dataset.
We split the dataset into a probing set and an evaluation set.
Model responses on the probing set are labeled as \emph{truthful} or \emph{hallucinated} using
\textbf{QwQ-32B} as an external judge.
The positive set $\mathcal{P}$ consists of truthful responses, and the negative set
$\mathcal{N}$ consists of hallucinated responses.

\paragraph{efficient Reasoning.}
For reasoning efficiency control, we use the GSM8K training set. We measure the length of each model-generated reasoning trace and select the top 5\% shortest and top 5\% longest responses. The positive set $\mathcal{P}$ consists of short reasoning traces, and the negative set
$\mathcal{N}$ consists of long reasoning traces. The resulting steering vectors capture directions associated with shorter internal reasoning processes.

All steering vectors are computed once per model and per behavioral control setting.
During evaluation, the same steering vectors are applied across all test sets, reflecting a practical deployment scenario in which vectors are not optimized for any specific evaluation benchmark.

\newpage

\section{Hyperparameter Search Procedure}
\label{app:hyperparam}

\textsc{\textbf{Steer2Edit}} introduces three scalar hyperparameters:
the attention editing budget $\rho_{\text{attn}}$, the MLP editing budget $\rho_{\text{mlp}}$,
and the Elastic-Net sparsity parameter $\alpha$.
These parameters control how edit magnitude is allocated across model components
and how sparsely edits are distributed.
Hyperparameters are selected using a lightweight two-stage grid search on held-out data.

\subsection*{Step 1: Coarse grid search}

We first perform a coarse-grained grid search over a shared range
that is identical across all models and behavioral control settings:
\[
\rho_{\text{attn}} \in \{0.1, 0.3, 0.5, 0.7, 0.9\}, \quad
\rho_{\text{mlp}} \in \{0.1, 0.3, 0.5, 0.7, 0.9\}, \quad
\alpha \in \{0.1, 0.3, 0.5, 0.7, 0.9\}.
\]
The goal of this step is to identify the approximate operating regime
(e.g., attention-dominated versus MLP-dominated edits,
sparse versus distributed allocation),
rather than to finely optimize performance.

Each configuration is first subjected to a lightweight sanity check
using 20 short, simple prompts.
If the edited model exhibits degenerate behavior
(e.g., repetitive output, failure to respond, or nonsensical generations),
the configuration is immediately discarded.
This allows unstable settings to be filtered at negligible cost.

\subsection*{Step 2: Refined grid search}

Based on the results of the coarse search, we define a refined but still small
grid for each (model, setting) pair.
The refined grids narrow the range and reduce the step size around regions
that exhibit meaningful improvements in the target attribute
while preserving normal model behavior.

In several settings, we observe that edits to either attention or MLP components
have negligible impact on the target behavior.
In these cases, the corresponding component is not edited at all,
and no budget is assigned to that component during the refined search.

\subsection*{Final search ranges}

Table~\ref{tab:hyperparam_grid} summarizes the refined hyperparameter ranges
used in each behavioral control setting.
All reported results in the main paper, including the best-performing configuration
and the top-10 configurations shown in trade-off plots, are selected
exclusively from these ranges.

\begin{table}[H]
\centering
\small
\setlength{\tabcolsep}{4pt}
\begin{tabular*}{\linewidth}{@{\extracolsep{\fill}} l l l l}
\toprule
\textbf{Setting / Model} & $\boldsymbol{\rho_{\text{attn}}}$ & $\boldsymbol{\rho_{\text{mlp}}}$ & $\boldsymbol{\alpha}$ \\
\midrule
\multicolumn{4}{l}{\emph{Safety Alignment}} \\
LLaMA-2-7B-Chat
& $[0.16,\,0.24]$ (step = 0.02)
& $[0.35,\,0.55]$ (step = 0.05)
& $[0.70,\,0.90]$ (step = 0.05) \\
Mistral-7B-Instruct-v0.2
& $[0.42,\,0.50]$ (step = 0.02)
& $[0.40,\,0.60]$ (step = 0.05)
& $[0.65,\,0.85]$ (step = 0.05) \\
\midrule
\multicolumn{4}{l}{\emph{Truthfulness}} \\
Gemma-2-2B-IT
& $[0.30,\,0.50]$ (step = 0.05)
& \textit{negligible}
& $[0.75,\,0.95]$ (step = 0.05) \\
LLaMA-3-8B-Instruct
& $[0.10,\,0.14]$ (step = 0.01)
& $[0.30,\,0.50]$ (step = 0.05)
& $[0.30,\,0.70]$ (step = 0.10) \\
\midrule
\multicolumn{4}{l}{\emph{Efficient Reasoning}} \\
Qwen3-4B-Thinking-2507
& \textit{negligible}
& $[0.65,\,0.80]$ (step = 0.05)
& $[0.05,\,0.20]$ (step = 0.05) \\
OpenMath-Nemotron-7B
& $[0.20,\,0.30]$ (step = 0.05)
& $[0.80,\,0.90]$ (step = 0.05)
& $[0.10,\,0.20]$ (step = 0.05) \\
\bottomrule
\end{tabular*}
\caption{Refined hyperparameter search ranges for \textsc{\textbf{Steer2Edit}}.
``Negligible'' indicates that edits to the corresponding component
were found to have insufficient effect during coarse search
and are therefore not applied in the refined search.}
\vspace{-10pt}
\label{tab:hyperparam_grid}
\end{table}

\subsection*{Efficiency and reporting}

Hyperparameter search for \textsc{\textbf{Steer2Edit}} is computationally lightweight.
Each configuration requires only a single closed-form application
of rank-1 weight edits, followed by evaluation on a held-out small validation set.
In practice, the full two-stage search completes within minutes per model, and does not involve gradient-based optimization. All evaluations are performed on held-out data that is disjoint
from steering vector extraction.
No additional tuning is performed on test sets.

\newpage

\section{Per-Dataset Trade-off Analysis}
\label{sec:appendix_tradeoff}

In Section \ref{subsec:usecases}, we summarize each behavioral control setting
using aggregated downstream utility metrics to provide a concise comparison across
methods.
In this appendix, we present \emph{per-dataset trade-off curves} that expose finer-grained
behavior across individual evaluation benchmarks.
All three activation-steering baselines (RepEngineering, CAA, ODESteer) are shown alongside \textsc{\textbf{Steer2Edit}}.
These results demonstrate that the superior attribute--utility trade-offs achieved by
\textsc{\textbf{Steer2Edit}} are consistent across datasets and across all baselines.

\subsection{Safety Alignment: Attack- and Dataset-Specific Trade-offs}
\label{sec:appendix_safety}

For safety alignment, we evaluate two jailbreak attack methods (GCG and ADV-LLM)
and three downstream utility benchmarks (CommonsenseQA, Code-MMLU, and GSM8K),
resulting in six distinct safety--utility trade-off settings per model.

Figure~\ref{fig:appendix_safety_per_dataset} reports refusal rate versus downstream
utility separately for each attack--dataset pair on \textbf{LLaMA-2-7B-Chat} and
\textbf{Mistral-7B-Instruct-v0.2}.
Across most settings, \textsc{\textbf{Steer2Edit}} identifies configurations that achieve higher
refusal rates at comparable or higher utility than all three activation-steering baselines.

For \textbf{LLaMA-2-7B-Chat}, \textsc{\textbf{Steer2Edit}} consistently dominates all baselines under both GCG and ADV-LLM attacks across all downstream datasets.
CAA and RepEngineering follow similar trajectories, while ODESteer achieves lower refusal rates at comparable utility costs.
For \textbf{Mistral-7B-Instruct-v0.2}, CAA reaches a trade-off region closer to \textsc{\textbf{Steer2Edit}} on some attack--dataset pairs, but this competitiveness is setting-specific and does not generalize to truthfulness or reasoning efficiency.

Notably, the advantage of \textsc{\textbf{Steer2Edit}} becomes more pronounced as the attack
strength increases.
While all activation-steering baselines require aggressive intervention that sharply degrades
utility under ADV-LLM, weight-level edits derived by \textsc{\textbf{Steer2Edit}} maintain
stable benign-task performance while substantially improving robustness to strong
jailbreaks.

\subsection{Truthfulness: Dataset-Specific Utility Trade-offs}
\label{sec:appendix_truthfulness}

For truthfulness promotion, downstream utility is evaluated independently on
CommonsenseQA, Code-MMLU, and GSM8K.
Figure~\ref{fig:appendix_truth_per_dataset} shows truthfulness versus utility
accuracy for \textbf{Gemma-2-2B-IT} and \textbf{LLaMA-3-8B-Instruct} on each benchmark.

Across all datasets, all three activation-steering baselines exhibit a pronounced trade-off in which
increasing truthfulness rapidly degrades task performance.
In contrast, \textsc{\textbf{Steer2Edit}} consistently attains higher truthfulness at
substantially higher utility, with edited configurations occupying regions of the
trade-off space that are unattainable by any steering baseline.

These per-dataset results demonstrate that the truthfulness gains of
\textsc{\textbf{Steer2Edit}} generalize across reasoning-heavy and knowledge-oriented
benchmarks.

\subsection{Efficient Reasoning: Dataset-Level Accuracy--Length Trade-offs}
\label{sec:appendix_efficiency}

For efficient reasoning control, we report dataset-specific trade-offs between answer
accuracy and reasoning length on GSM8K, MATH-500, GPQA, and Code-MMLU for
\textbf{Qwen3-4B-Thinking-2507} and \textbf{OpenMath-Nemotron-7B}.

Figure~\ref{fig:appendix_efficiency_per_dataset} shows that all three activation-steering baselines reduce
reasoning length primarily by sacrificing accuracy, with the severity of this trade-off
varying substantially across datasets.
In contrast, \textsc{\textbf{Steer2Edit}} consistently identifies configurations that shorten
reasoning traces while preserving accuracy, dominating the Pareto frontier of all baselines on every dataset including challenging benchmarks such as GPQA.

Notably, these improvements generalize beyond GSM8K, despite the steering direction
being extracted from GSM8K, indicating that \textsc{\textbf{Steer2Edit}} captures a transferable
mechanism for reasoning efficiency control.

\begin{figure}[H]
    \centering
    \includegraphics[width=\linewidth]{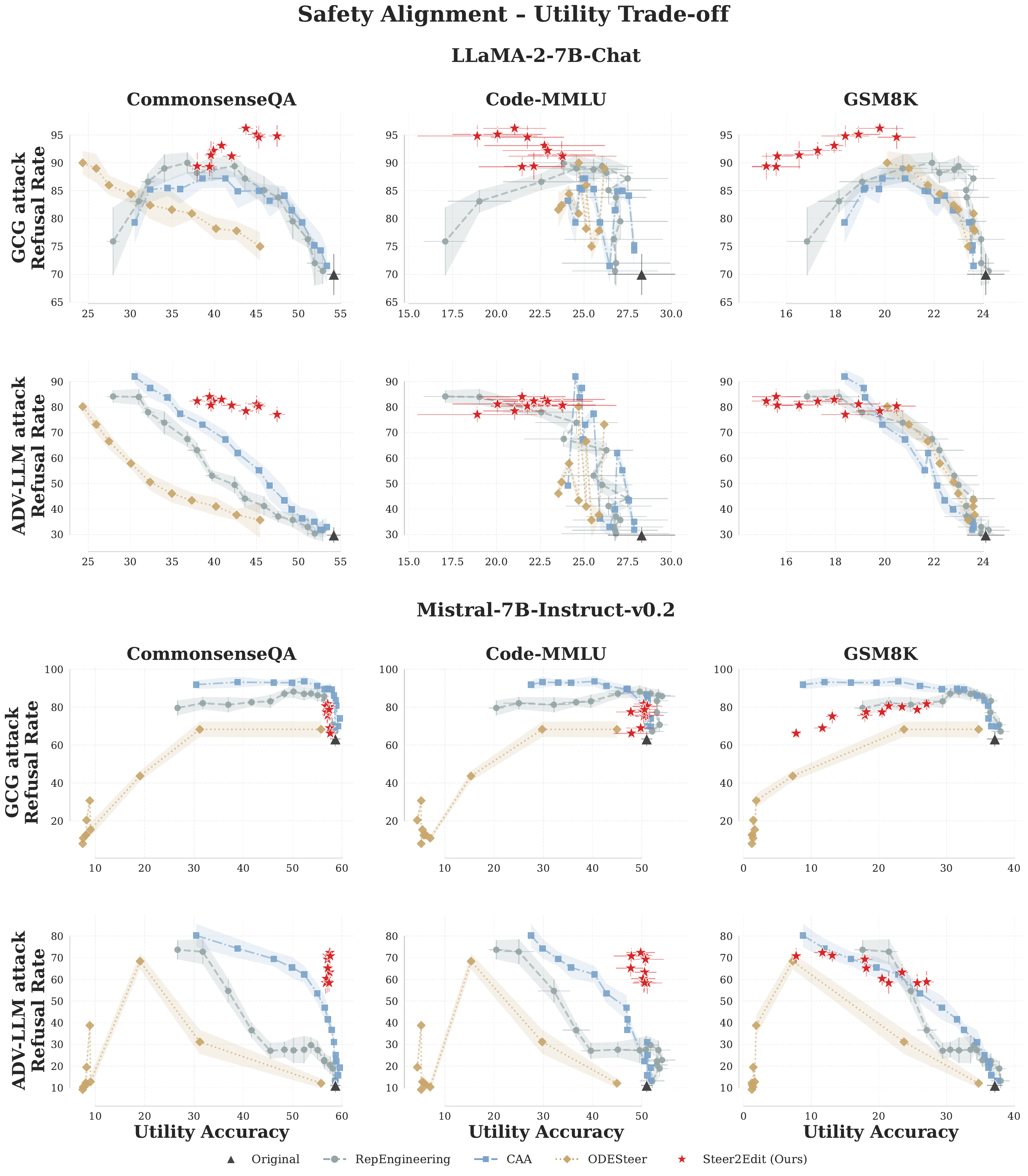}
    \caption{
    Per-dataset safety--utility trade-offs under GCG and ADV-LLM attacks.
    Each column corresponds to a downstream utility dataset
    (CommonsenseQA, Code-MMLU, GSM8K), and each row corresponds to an attack method.
    Three activation-steering baselines (RepEngineering, CAA, ODESteer) are shown alongside \textsc{\textbf{Steer2Edit}}.
    \textsc{\textbf{Steer2Edit}} consistently achieves higher refusal rates at comparable or
    higher utility than all baselines across most settings.
    }
    \label{fig:appendix_safety_per_dataset}
\end{figure}

\begin{figure}[H]
    \centering
    \includegraphics[width=\linewidth]{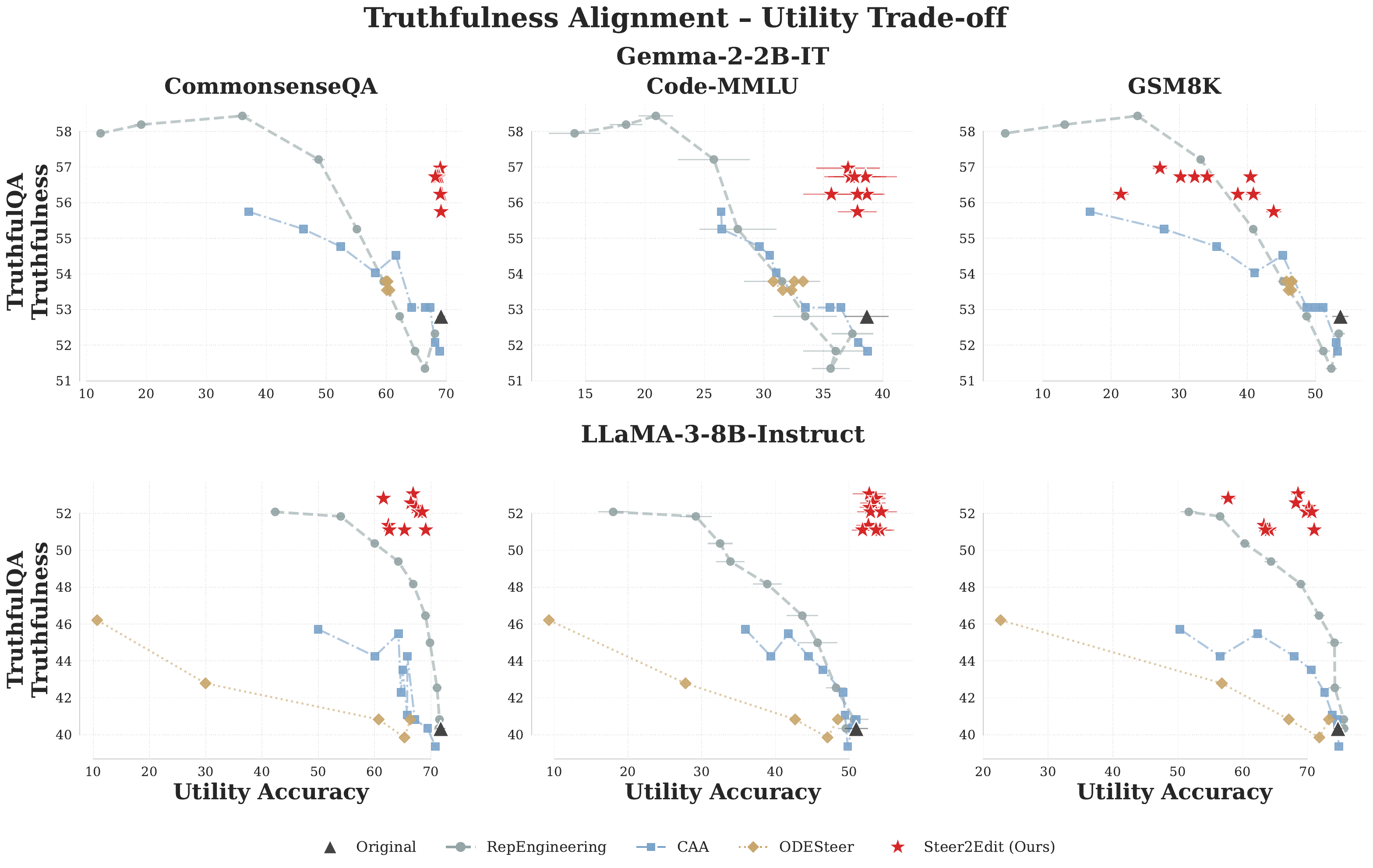}
    \caption{
    Per-dataset truthfulness--utility trade-offs on CommonsenseQA, Code-MMLU,
    and GSM8K for Gemma-2-2B-IT and LLaMA-3-8B-Instruct.
    Three activation-steering baselines (RepEngineering, CAA, ODESteer) are shown alongside \textsc{\textbf{Steer2Edit}}.
    \textsc{\textbf{Steer2Edit}} improves truthfulness while preserving higher downstream
    utility than all baselines across all datasets.
    }
    \label{fig:appendix_truth_per_dataset}
\end{figure}

\begin{figure}[H]
    \centering
    \includegraphics[width=\linewidth]{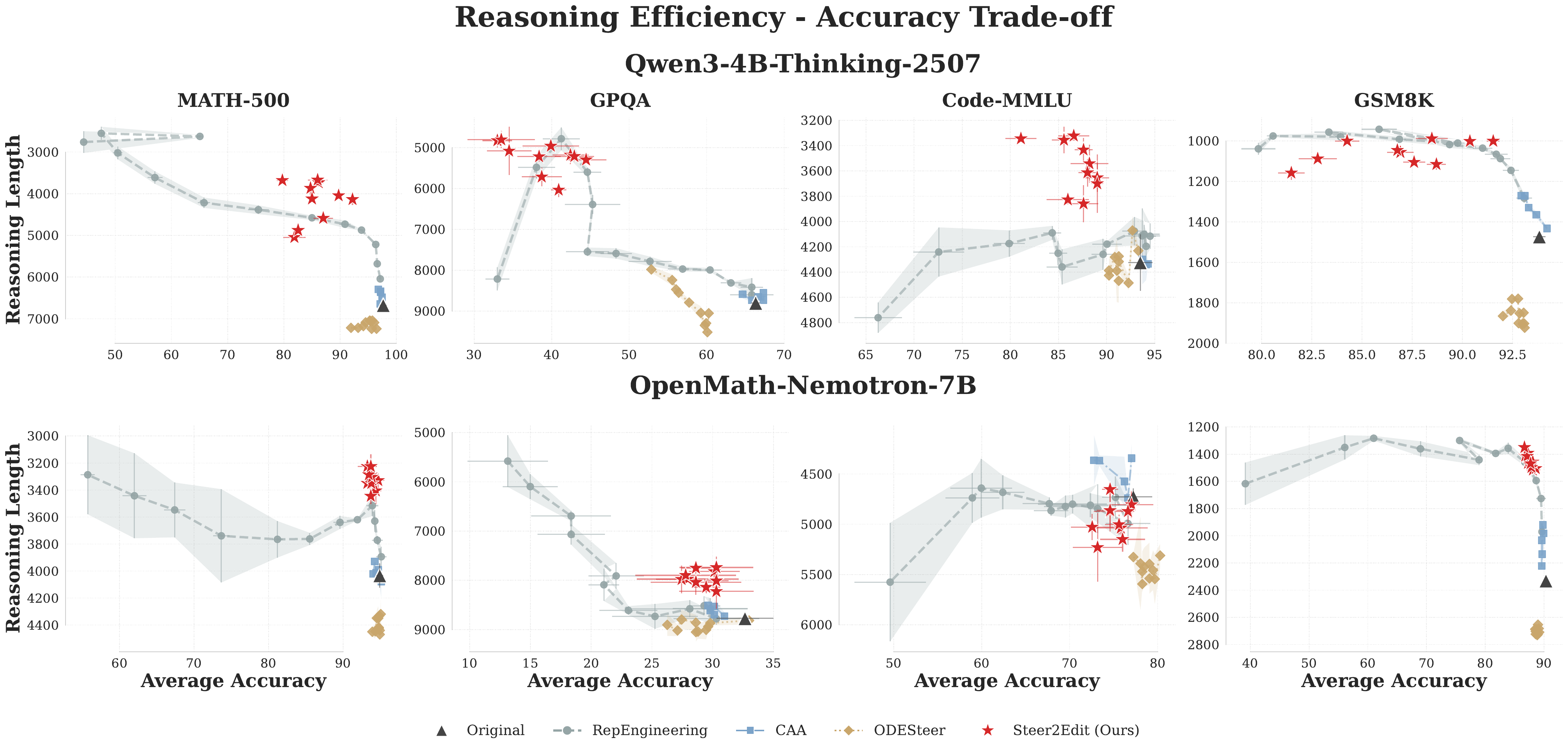}
    \caption{
    Per-dataset accuracy--reasoning-length trade-offs on GSM8K, MATH-500,
    GPQA, and Code-MMLU for Qwen3-4B-Thinking-2507 and OpenMath-Nemotron-7B.
    Three activation-steering baselines (RepEngineering, CAA, ODESteer) are shown alongside \textsc{\textbf{Steer2Edit}}.
    \textsc{\textbf{Steer2Edit}} reduces reasoning length while maintaining accuracy,
    dominating all baselines across all datasets.
    }
    \label{fig:appendix_efficiency_per_dataset}
\end{figure}

\newpage

\section{Ablation Study of \textsc{\textbf{Steer2Edit}}: Empirical Justification of Formal Design Choices}\label{app:ablation}

In this section, we validate the design choices of \textsc{\textbf{Steer2Edit}} through a systematic ablation study. We decompose the rank-1 weight update, $\Delta W_i = \lambda_i u_i k_i^\top$, and independently modify its three core components: the \textbf{Input Direction} ($k$), the \textbf{Importance Score} ($g$), and the \textbf{Edit Magnitude} ($\lambda$), while holding the others constant.

This study aims to determine whether the performance of \textsc{\textbf{Steer2Edit}} arises from its specific geometric and statistical formulations—such as using cosine similarity for scoring or Elastic-Net for sparsity—rather than generic weight modifications. Table~\ref{tab:ablation_definitions} defines the five ablation variants tested and the specific hypothesis each one investigates.

\begin{table}[H]
\centering
\caption{\textbf{Ablation Definitions.} We categorize variants by the component they modify: Input Direction ($k$), Importance Score ($g$), or Magnitude ($\lambda$).}
\label{tab:ablation_definitions}
\vspace{0.5em}
\small
\setlength{\tabcolsep}{8pt}
\begin{tabular*}{\textwidth}{@{\extracolsep{\fill}} l l p{0.6\textwidth} @{}}
\toprule
\textbf{Ablation Category} & \textbf{Variant} & \textbf{Definition and Hypothesis Tested} \\
\midrule
\multirow{2}{*}{\textbf{Input} $k$} & $k_{\text{mean}}$ & Sets $k_i \leftarrow \mu_i$, where $\mu_i = \mathbb{E}[h_i]$. Tests if data statistics alone suffice without directional sensitivity. \\
 & $k_{\text{svd}}$ & Sets $k_i \leftarrow \mathbf{v}_1$, where $\mathbf{v}_1$ is the top right singular vector of $W_i$. Tests if intrinsic weight directions suffice. \\
\midrule
\textbf{Score} $g$ & $g_{\text{dot}}$ & Sets $g_i \leftarrow \hat{v}_i^\top (W_i \mu_i)$ (unnormalized dot product), removing the normalization by $\|W_i \mu_i\|_2$. Tests the effect of component-output normalization. \\

\midrule
\multirow{2}{*}{\textbf{Magnitude} $\lambda$} & $\ell_0$ & Selects top-$K$ components (hard threshold) matching Elastic-Net sparsity. Tests if sparsity alone explains gains. \\
 & $\ell_2$ & Uses Elastic-Net with $\alpha=0$ (Ridge regularization), resulting in dense edits. Tests if dense edits can preserve utility. \\
\bottomrule
\end{tabular*}
\end{table}

\subsection{Unified Performance Summary}

To facilitate a high-level comparison, Table~\ref{tab:unified_summary} aggregates the results across all three behavioral settings: Safety Alignment, Truthfulness, and Efficient Reasoning.

We normalize the attribute scores to a common $[0,1]$ scale. For \textbf{Safety} and \textbf{Truthfulness}, we use the raw percentage divided by 100. For \textbf{Efficient Reasoning}, where lower length is better, we define the efficiency score as $\min(1, L_{\textsc{\textbf{Steer2Edit}}} / L_{\text{Ablation}})$. \textsc{\textbf{Steer2Edit}} consistently achieves the highest combined \textbf{Attribute $\times$ Utility} score, demonstrating that precise input alignment, normalized scoring, and sparse regularization are all critical for optimal performance.

\begin{table}[H]
\centering
\caption{\textbf{Unified Performance Summary.} Results are averaged across 6 models (2 per setting). \textsc{\textbf{Steer2Edit}} achieves the best global trade-off between targeting the attribute and maintaining downstream utility.}
\label{tab:unified_summary}
\vspace{0.5em}
\scriptsize
\setlength{\tabcolsep}{6pt}
\begin{tabular*}{\textwidth}{@{\extracolsep{\fill}} l l | ccc | ccc | cc | c @{}}
\toprule
\multirow{2}{*}{\textbf{Category}} & \multirow{2}{*}{\textbf{Variant}} & \multicolumn{3}{c|}{\textbf{Normalized Attribute Score} $\uparrow$} & \multicolumn{3}{c|}{\textbf{Normalized Utility Score} $\uparrow$} & \multicolumn{2}{c|}{\textbf{Overall Average} $\uparrow$} & \multirow{2}{*}{\textbf{Attr} $\times$ \textbf{Util}} \\
\cmidrule(lr){3-5} \cmidrule(lr){6-8} \cmidrule(lr){9-10}
 & & Safety & Truth & Effic. & Safety & Truth & Effic. & \textbf{Attribute} & \textbf{Utility} & \\
\midrule
\textbf{Full Method} & \textbf{Steer2Edit} & 0.807 & \textbf{0.550} & \textbf{1.000} & 0.341 & 0.536 & 0.755 & \textbf{0.786} & 0.544 & \textbf{0.427} \\
\midrule
\multirow{2}{*}{\textbf{Input} $k$} & $k_{\text{mean}}$ & \textbf{0.827} & 0.531 & 0.926 & 0.279 & 0.350 & 0.589 & 0.761 & 0.406 & 0.309 \\
 & $k_{\text{svd}}$ & 0.536 & 0.474 & 0.757 & \textbf{0.400} & \textbf{0.572} & \textbf{0.805} & 0.589 & \textbf{0.592} & 0.349 \\
\midrule
\textbf{Score} $g$ & $g_{\text{dot}}$ & 0.102 & 0.509 & 0.628 & 0.047 & 0.232 & 0.003 & 0.413 & 0.094 & 0.039 \\
\midrule
\multirow{2}{*}{\textbf{Magnitude} $\lambda$} & $\ell_0$ & 0.506 & 0.546 & 0.767 & 0.321 & 0.423 & 0.596 & 0.606 & 0.447 & 0.271 \\
 & $\ell_2$ & 0.000 & 0.580 & 0.472 & 0.032 & 0.226 & 0.394 & 0.350 & 0.217 & 0.076 \\
\bottomrule
\end{tabular*}
\end{table}

\paragraph{Key Insights: Normalization and Sparsity.}
Two critical design principles emerge from these aggregated results. First, \textbf{Score Normalization is paramount for stability.} The catastrophic failure of the unnormalized $g_{\text{dot}}$ variant (Overall Score: 0.039) reveals that raw activation magnitudes vary drastically across model layers. Without the cosine-similarity normalization used in \textsc{\textbf{Steer2Edit}}, the editing process becomes dominated by high-norm layers, destabilizing the model regardless of the target attribute. Second, \textbf{Sparsity is essential for utility preservation.} The dense $\ell_2$ baseline struggles to maintain downstream performance (Utility: 0.217), confirming that precise, sparse interventions are required to disentangle specific behaviors without overwriting the model's general knowledge base.

\subsection{Detailed Per-Setting Results}

The following subsections analyze the impact of each component across our three behavioral settings. We find that while different tasks exhibit unique sensitivities, the failures of the ablation baselines consistently point to the necessity of \textsc{\textbf{Steer2Edit}}'s three-pillared approach: precise directional alignment, normalized scoring, and sparse editing.

\paragraph{Safety Alignment: Selective vs. Uniform Triggering.}
Table~\ref{tab:ablation_safety_detailed} highlights the risks of uniformly triggered
edits.
The $k_{\text{mean}}$ baseline, which activates edits based on the global mean activation
$\mathbb{E}[x]$, achieves high safety (e.g., 92.20\% refusal on ADV-LLM) but substantially
reduces utility (21.75\% vs.\ 28.00\% for \textsc{\textbf{Steer2Edit}}).
This indicates that activating edits in a largely input-agnostic manner leads to
significant degradation on benign downstream tasks.
Conversely, the dense $\ell_2$ baseline results in near-total model failure (0.05\%
safety), suggesting that safety-related behavior is mediated by localized components and
is disrupted by dense parameter modifications.
\textsc{\textbf{Steer2Edit}} avoids these failure modes by combining input-selective
activation with sparse, component-level edits.

\begin{table}[H]
\centering
\caption{\textbf{Safety Alignment Ablations (Detailed).} Metrics are Average Safety (Refusal Rate) and Average Utility (Accuracy).}
\label{tab:ablation_safety_detailed}
\footnotesize
\setlength{\tabcolsep}{4pt}
\begin{tabular*}{\textwidth}{@{\extracolsep{\fill}} l l c cc | cc | ccc @{}}
\toprule
\multirow{2}{*}{\textbf{Model}} & \multirow{2}{*}{\textbf{Variant}} & \multirow{2}{*}{\textbf{S $\times$ U}} & \textbf{Avg} & \textbf{Avg} & \multicolumn{2}{c|}{\textbf{Safety (Refusal Rate)}} & \multicolumn{3}{c}{\textbf{Utility (Accuracy)}} \\
 & & & \textbf{Safety} & \textbf{Util} & GCG & ADV-LLM & CommonsenseQA & Code-MMLU & GSM8K \\
\midrule
\multirow{6}{*}{\shortstack[l]{Llama2-\\7B-Chat}} 
 & Steer2Edit & \textbf{24.68} & \textbf{88.15} & 28.00 & \textbf{95.10} & 81.20 & 45.01 & 20.06 & 18.93 \\
 \cmidrule{2-10}
 & $k_{\text{mean}}$ & 19.12 & 87.90 & 21.75 & 83.60 & \textbf{92.20} & 30.14 & 19.57 & 15.53 \\
 & $k_{\text{svd}}$ & 22.82 & 66.35 & \textbf{34.39} & 90.80 & 41.90 & \textbf{50.77} & \textbf{30.49} & \textbf{21.90} \\
 \cmidrule{2-10}
 & $g_{\text{dot}}$ & 0.69 & 15.35 & 4.49 & 23.30 & 7.40 & 6.45 & 5.37 & 1.65 \\
 \cmidrule{2-10}
 & $\ell_0$ & 7.66 & 41.25 & 18.58 & 36.40 & 46.10 & 27.26 & 25.43 & 3.05 \\
 & $\ell_2$ & 0.00 & 0.05 & 6.33 & 0.10 & 0.00 & 3.60 & 14.88 & 0.52 \\
\bottomrule
\toprule
\multirow{6}{*}{\shortstack[l]{Mistral-7B\\-Instruct}} 
 & Steer2Edit & \textbf{29.38} & 73.15 & 40.16 & 75.20 & 71.10 & \textbf{57.00} & 50.37 & 13.10 \\
 \cmidrule{2-10}
 & $k_{\text{mean}}$ & 26.33 & \textbf{77.45} & 34.00 & 71.60 & \textbf{83.30} & 54.24 & 45.55 & 2.22 \\
 & $k_{\text{svd}}$ & 18.59 & 40.85 & 45.51 & 68.80 & 12.90 & 53.61 & \textbf{51.95} & 30.98 \\
 \cmidrule{2-10}
 & $g_{\text{dot}}$ & 0.25 & 5.05 & 4.96 & 0.90 & 9.20 & 5.43 & 9.45 & 0.01 \\
 \cmidrule{2-10}
 & $\ell_0$ & 27.29 & 59.85 & \textbf{45.60} & \textbf{77.20} & 42.50 & 55.96 & 48.35 & \textbf{32.48} \\
 & $\ell_2$ & 0.00 & 0.00 & 0.00 & 0.00 & 0.00 & 0.00 & 0.00 & 0.00 \\
\bottomrule
\end{tabular*}
\end{table}

\newpage

\paragraph{Truthfulness: The Critical Role of Sparsity.}
Table~\ref{tab:ablation_truth_detailed} provides strong evidence for the necessity of
sparse editing ($\ell_0$ regularization).
The dense $\ell_2$ variant, which modifies all parameters in the target block, causes
catastrophic utility collapse (0.00\% on Llama3-8B), demonstrating that dense edits
severely disrupt the model’s general capabilities.
Additionally, we observe an informative trade-off with $k_{\text{svd}}$: while it
preserves high utility (65.11\% on Llama3), it fails to significantly improve
truthfulness (39.61\%).
This suggests that aligning edits with the model’s intrinsic dominant activation
patterns is insufficient to induce a shift toward truthful behavior.
Only \textsc{\textbf{Steer2Edit}} balances these objectives: it uses $k$ to selectively
activate edits on semantically relevant inputs, while $\ell_0$ sparsity restricts the
intervention to a small set of behaviorally relevant neurons.

\begin{table}[H]
\centering
\caption{\textbf{Truthfulness Ablations (Detailed).} Metrics are Average Truthfulness (TruthfulQA Accuracy) and Average Utility (Accuracy).}
\label{tab:ablation_truth_detailed}
\small
\setlength{\tabcolsep}{4pt}
\begin{tabular*}{\textwidth}{@{\extracolsep{\fill}} l l c cc | c | ccc @{}}
\toprule
\multirow{2}{*}{\textbf{Model}} & \multirow{2}{*}{\textbf{Method}} & \multirow{2}{*}{\textbf{T $\times$ U}} & \textbf{Avg} & \textbf{Avg} & \textbf{Attribute} & \multicolumn{3}{c}{\textbf{Utility (Accuracy)}} \\
 & & & \textbf{Truth} & \textbf{Util} & TruthfulQA & CommonsenseQA & Code-MMLU & GSM8K \\
\midrule
\multirow{6}{*}{\shortstack[l]{Gemma2\\2B-IT}} 
 & Steer2Edit & 25.30 & 56.97 & 44.41 & 56.97 & \textbf{69.00} & 37.07 & 27.16 \\
 \cmidrule{2-9}
 & $k_{\text{mean}}$ & 24.93 & 55.75 & 44.71 & 55.75 & 68.59 & 38.29 & 27.24 \\
 & $k_{\text{svd}}$ & \textbf{27.26} & 55.26 & \textbf{49.33} & 55.26 & 68.12 & \textbf{39.39} & \textbf{40.48} \\
 \cmidrule{2-9}
 & $g_{\text{dot}}$ & 1.16 & \textbf{58.68} & 1.98 & \textbf{58.68} & 2.83 & 2.20 & 0.90 \\
 \cmidrule{2-9}
 & $\ell_0$ & 15.33 & 57.95 & 26.46 & 57.95 & 58.29 & 16.04 & 5.06 \\
 & $\ell_2$ & 26.15 & 57.95 & 45.12 & 57.95 & 63.62 & 36.95 & 34.78 \\
\bottomrule
\toprule
\multirow{6}{*}{\shortstack[l]{Llama3\\8B-Instruct}} 
 & Steer2Edit & \textbf{33.27} & 53.06 & 62.71 & 53.06 & 66.86 & \textbf{52.74} & 68.54 \\
 \cmidrule{2-9}
 & $k_{\text{mean}}$ & 12.74 & 50.37 & 25.29 & 50.37 & 39.67 & 25.61 & 10.60 \\
 & $k_{\text{svd}}$ & 25.79 & 39.61 & \textbf{65.11} & 39.61 & \textbf{71.30} & 51.71 & \textbf{72.32} \\
 \cmidrule{2-9}
 & $g_{\text{dot}}$ & 19.11 & 43.03 & 44.42 & 43.03 & 26.94 & 41.77 & 64.56 \\
 \cmidrule{2-9}
 & $\ell_0$ & 29.87 & 51.34 & 58.18 & 51.34 & 64.14 & 47.13 & 63.26 \\
 & $\ell_2$ & 0.00 & \textbf{57.95} & 0.00 & \textbf{57.95} & 0.00 & 0.00 & 0.00 \\
\bottomrule
\end{tabular*}
\end{table}

\paragraph{Efficient Reasoning: Stability via Normalization.}
Table~\ref{tab:ablation_reasoning_detailed} highlights the critical role of score
normalization for stable reasoning control.
The $g_{\text{dot}}$ variant, which uses the raw dot product, fails pathologically: it
reduces the reasoning length to just 21 tokens (Qwen3-4B) while collapsing utility to
near zero (0.13\%).
This behavior arises because activation norms vary substantially across layers; without
the cosine normalization used in \textsc{\textbf{Steer2Edit}}, edit magnitudes become
dominated by high-norm layers, leading to severe disruption of model behavior.
Conversely, the $k_{\text{svd}}$ baseline increases reasoning length (5351 vs.\ 3467 for
\textsc{\textbf{Steer2Edit}}), indicating that activating edits along the model’s
intrinsic dominant activation patterns is insufficient for improving efficiency.
Together, these results show that effective reasoning-length control requires both
normalized importance scoring and task-specific edit activation, as implemented in
\textsc{\textbf{Steer2Edit}}.

\begin{table}[H]
\centering
\caption{\textbf{Efficient Reasoning Ablations (Detailed).} Metrics are Reasoning Length (Lower is Better) and Utility (Higher is Better). The \textbf{U / L} metric represents efficiency ($\text{Utility} / \text{Length} \times 100$).}
\label{tab:ablation_reasoning_detailed}
\tiny
\setlength{\tabcolsep}{4pt}
\begin{tabular*}{\textwidth}{@{\extracolsep{\fill}} l l c cc | cccc | cccc @{}}
\toprule
\multirow{2}{*}{\textbf{Model}} & \multirow{2}{*}{\textbf{Variant}} & \multirow{2}{*}{\textbf{U / L}} & \textbf{Avg} & \textbf{Avg} & \multicolumn{4}{c|}{\textbf{Reasoning Length (Lower is Better)}} & \multicolumn{4}{c}{\textbf{Utility (Higher is Better)}} \\
 & & & \textbf{Len} & \textbf{Util} & MATH-500 & GPQA & Code-MMLU & GSM8K & MATH-500 & GPQA & Code-MMLU & GSM8K \\
\midrule
\multirow{6}{*}{\shortstack[l]{Qwen3\\4B-Thinking}} 
 & Steer2Edit & \textbf{2.28} & 3467 & 78.95 & 4136 & 5299 & 3433 & 1000 & 92.2 & 44.4 & 87.6 & 91.5 \\
 \cmidrule{2-13}
 & $k_{\text{mean}}$ & 1.56 & 3266 & 51.05 & 3188 & 2728 & 6129 & 1018 & 53.2 & 10.3 & 78.1 & 62.7 \\
 & $k_{\text{svd}}$ & 1.64 & 5351 & \textbf{87.99} & 6389 & 8780 & 4622 & 1613 & \textbf{96.9} & \textbf{67.5} & \textbf{93.9} & \textbf{93.6} \\
 \cmidrule{2-13}
 & $g_{\text{dot}}$ & 0.62 & \textbf{21} & 0.13 & \textbf{23} & \textbf{33} & \textbf{10} & \textbf{18} & 0.3 & 0.0 & 0.0 & 0.2 \\
 \cmidrule{2-13}
 & $\ell_0$ & 0.73 & 6506 & 47.61 & 8646 & 5202 & 4698 & 7477 & 44.7 & 36.0 & 65.7 & 44.1 \\
 & $\ell_2$ & 0.08 & 22476 & 17.63 & 32223 & 18987 & 6420 & 32274 & 3.0 & 15.2 & 50.4 & 2.0 \\
\bottomrule
\toprule
\multirow{6}{*}{\shortstack[l]{OpenMath-\\Nemotron-7B}} 
 & Steer2Edit & 1.62 & 4445 & 72.12 & 3405 & \textbf{7821} & 5038 & \textbf{1515} & 94.4 & 30.1 & \textbf{75.8} & 88.1 \\
 \cmidrule{2-13}
 & $k_{\text{mean}}$ & 1.28 & 5216 & 66.69 & 3838 & 8814 & 6522 & 1691 & 92.5 & 25.8 & 66.7 & 81.9 \\
 & $k_{\text{svd}}$ & 1.42 & 5131 & \textbf{73.01} & 4188 & 8880 & 4985 & 2472 & 93.9 & \textbf{32.8} & 75.6 & \textbf{89.7} \\
 \cmidrule{2-13}
 & $g_{\text{dot}}$ & 0.00 & 17338 & 0.52 & 18380 & 16628 & 17163 & 17182 & 0.9 & 0.0 & 0.6 & 0.6 \\
 \cmidrule{2-13}
 & $\ell_0$ & \textbf{1.64} & \textbf{4356} & 71.60 & 3422 & 7971 & \textbf{4503} & 1528 & \textbf{94.5} & 29.6 & 73.4 & 88.9 \\
 & $\ell_2$ & 1.09 & 5633 & 61.23 & \textbf{2911} & 12686 & 5160 & 1773 & 89.5 & 17.7 & 66.1 & 71.7 \\
\bottomrule
\end{tabular*}
\end{table}

\newpage

\section{Component-Wise Budget Sensitivity Analysis}
\label{app:component_budget}

In the main paper, we show that the best-performing \textsc{\textbf{Steer2Edit}} configurations exhibit a consistent component-level structure: safety and truthfulness control rely on sparse edits to attention heads, whereas reasoning efficiency is primarily governed by distributed MLP neuron edits.

To further validate that this structural separation is intrinsic to the underlying mechanisms, we conduct a controlled \emph{component-wise budget sensitivity analysis}. In this study, we fix the sparsity parameter $\alpha$ and vary the regularization budget of one component class at a time, while disabling edits to the other class by sending its budget to infinity ($\rho \to \infty$). This isolates how changes in attention and MLP budgets individually influence the attribute--utility trade-off.

\subsection{Safety Alignment: Sensitivity to Attention Budget}

Figure~\ref{fig:safety_budget_sensitivity} illustrates how the safety--utility trade-off responds to changes in the attention and MLP budgets when considered in isolation. Increasing the attention budget $\rho_{\text{attn}}$ produces substantial gains in refusal rate at moderate utility cost, closely matching the best joint configurations reported in the main paper. In contrast, varying the MLP budget $\rho_{\text{mlp}}$ leads to markedly weaker safety improvements and often degrades utility more rapidly.

This asymmetric sensitivity indicates that safety alignment is primarily mediated by a small number of attention heads. The result is consistent with the sparsity patterns observed in Figure~\ref{fig:safety_edit_distribution}, where non-zero edit coefficients are concentrated in late-layer attention components, with minimal contribution from MLP neurons.

\begin{figure}[H]
    \centering
    \includegraphics[width=0.8\linewidth]{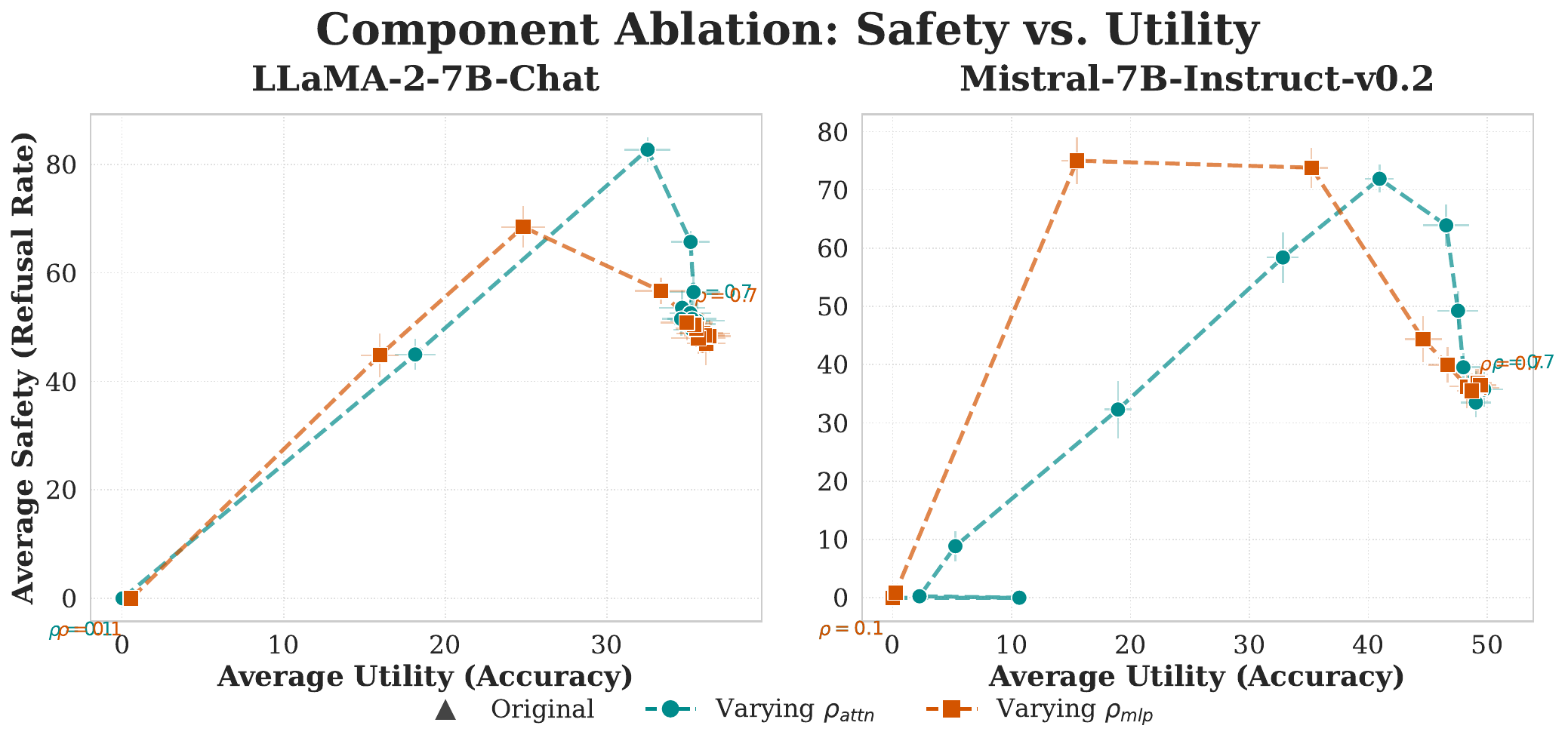}
    \caption{
    \textbf{Component-wise budget sensitivity for safety alignment.}
    We fix the sparsity parameter $\alpha$ and vary the attention regularization budget $\rho_{\text{attn}}$ while disabling MLP edits by taking $\rho_{\text{mlp}} \to \infty$, and vice versa. Improvements in refusal rate are primarily driven by changes in the attention budget, whereas varying the MLP budget yields limited safety gains and earlier utility degradation.
    }
    \label{fig:safety_budget_sensitivity}
\end{figure}

\subsection{Truthfulness: Sensitivity Dominated by Attention}

Figure~\ref{fig:truth_budget_sensitivity} reports the component-wise budget sensitivity for truthfulness promotion. Across both evaluated models, increasing the attention budget consistently yields larger improvements in truthful preference accuracy than increasing the MLP budget under the same sparsity constraint. MLP-only edits fail to recover the trade-off frontier achieved by attention edits.

These findings align with the component-level edit distributions shown in Figure~\ref{fig:truth_edit_distribution}, which reveal that truthfulness control is achieved through sparse, localized attention interventions. Notably, several models exhibit predominantly negative edit coefficients, suggesting that suppressing hallucination-promoting attention heads is more effective than broadly modifying MLP computation.

\begin{figure}[H]
    \centering
    \includegraphics[width=0.8\linewidth]{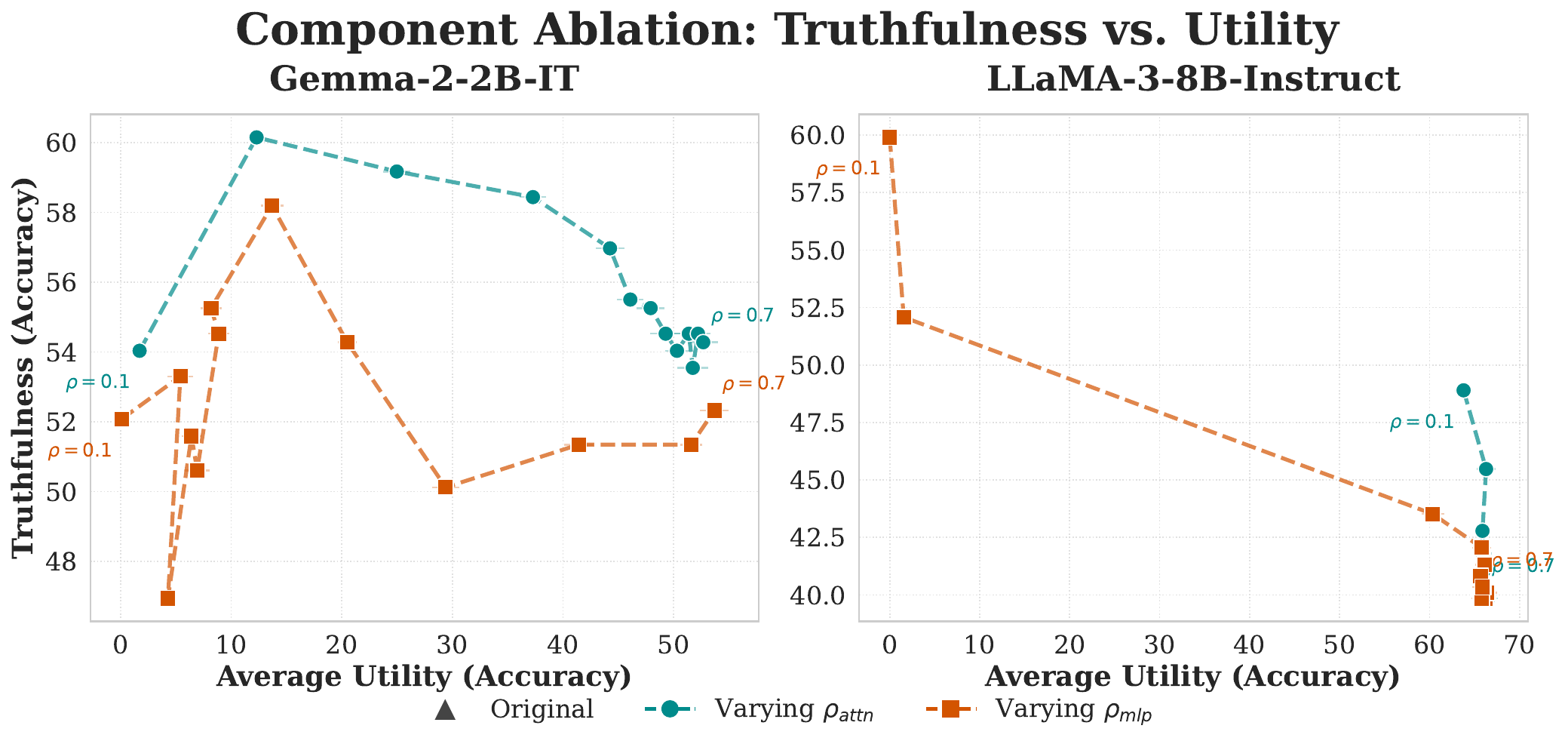}
    \caption{
    \textbf{Component-wise budget sensitivity for truthfulness control.}
    Truthfulness improvements are strongly sensitive to the attention budget $\rho_{\text{attn}}$, while varying the MLP budget $\rho_{\text{mlp}}$ in isolation results in substantially smaller gains at comparable downstream utility.
    }
    \label{fig:truth_budget_sensitivity}
\end{figure}

\subsection{Efficient Reasoning: Sensitivity Dominated by MLP Budget}

Figure~\ref{fig:efficiency_budget_sensitivity} shows that reasoning efficiency exhibits a qualitatively different sensitivity pattern. Increasing the MLP budget $\rho_{\text{mlp}}$ leads to a smooth and substantial reduction in reasoning length while preserving accuracy. In contrast, varying the attention budget $\rho_{\text{attn}}$ in isolation produces only marginal efficiency improvements, even at large budgets.

This behavior mirrors the dense MLP edit patterns observed in Figure~\ref{fig:reasoning_edit_distribution} and confirms that efficient reasoning control requires coordinated, distributed modifications to MLP neurons across layers. Unlike safety and truthfulness, which are governed by localized attention-based circuits, reasoning efficiency emerges from broad MLP-based computation.

\begin{figure}[H]
    \centering
    \includegraphics[width=0.8\linewidth]{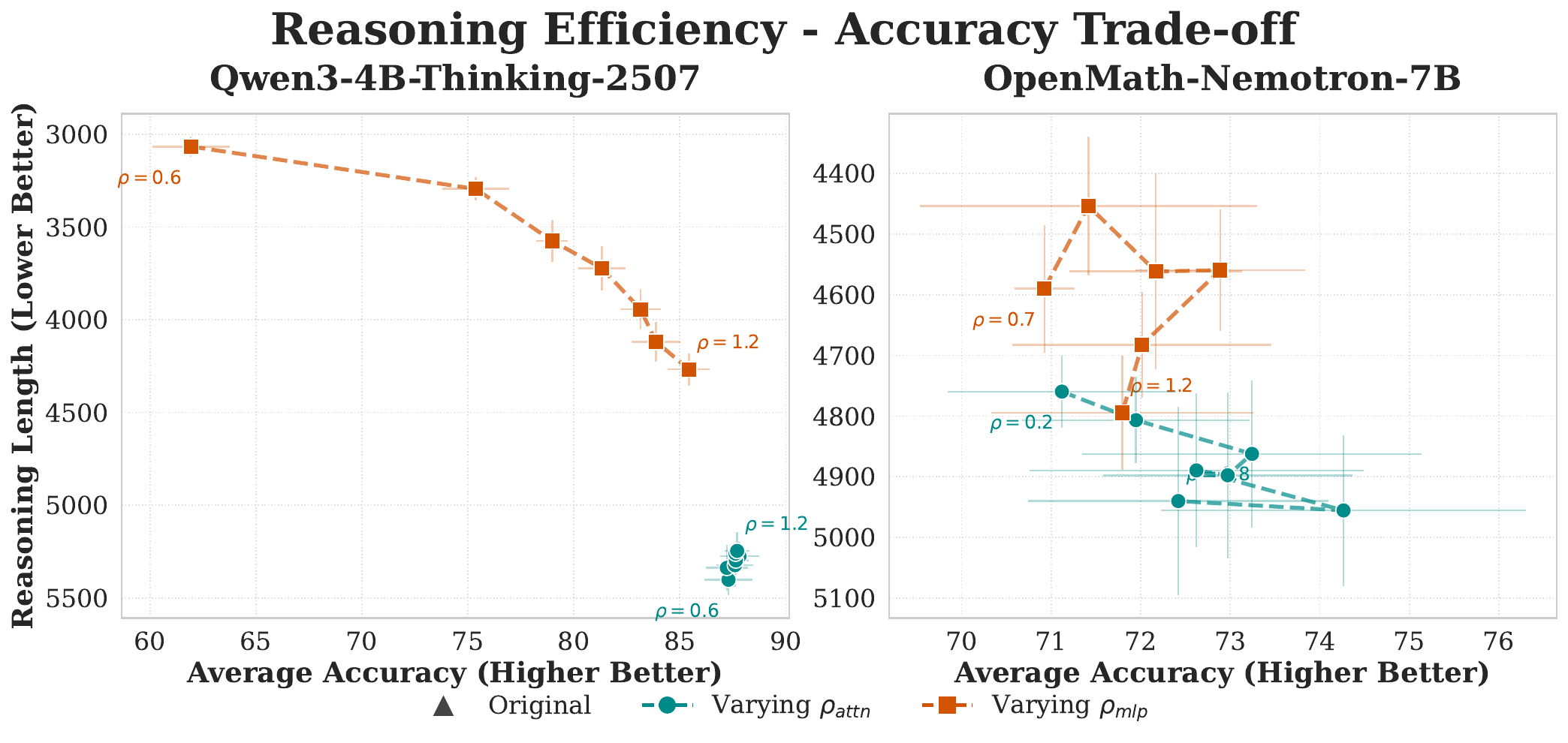}
    \caption{
    \textbf{Component-wise budget sensitivity for reasoning efficiency.}
    Reductions in reasoning length are strongly influenced by the MLP budget $\rho_{\text{mlp}}$, while varying the attention budget $\rho_{\text{attn}}$ yields only minor efficiency gains. This indicates that efficient reasoning is governed by distributed MLP computation rather than sparse attention circuits.
    }
    \label{fig:efficiency_budget_sensitivity}
\end{figure}

\paragraph{Summary.}
Across all behavioral control settings, this component-wise budget sensitivity analysis establishes a clear correspondence between the component class whose budget most strongly influences the trade-off frontier and the components receiving non-zero edits in the best-performing \textsc{\textbf{Steer2Edit}} configurations. Safety and truthfulness are attention-dominated, whereas reasoning efficiency is MLP-dominated, providing further evidence that \textsc{\textbf{Steer2Edit}} uncovers genuine, setting-dependent circuit structure rather than artifacts of hyperparameter tuning.

\newpage

\section{Additional Baselines: Comparing \textsc{Steer2Edit} with Training-Based Methods}
\label{app:finetune_baselines}

We compare \textsc{Steer2Edit} against training-based adaptation methods that directly optimize model parameters toward a target behavior.
Specifically, we consider \textbf{full-parameter fine-tuning} and \textbf{rank-1 LoRA fine-tuning} as additional baselines.

\paragraph{Setup and comparison protocol.}
For each control setting (safety, truthfulness, and efficient reasoning), we fine-tune models on the \emph{same probing dataset used to extract steering vectors}.
Training uses the \emph{positive set} (i.e., examples that exhibit the target attribute), so the model is explicitly optimized to imitate the desired behavior.
We evaluate the resulting models using the same attribute and downstream utility metrics as in the main experiments and report trade-off curves alongside activation steering and \textsc{Steer2Edit}.
The \textbf{full} fine-tuning baseline updates all model parameters.
The \textbf{rank-1} baseline applies LoRA adapters with rank $r=1$, inserted into the standard attention projections (\texttt{q\_proj}, \texttt{k\_proj}, \texttt{v\_proj}, \texttt{o\_proj}) and MLP projections (\texttt{gate\_proj}, \texttt{up\_proj}, \texttt{down\_proj}), while keeping the backbone weights frozen.
All baselines are trained with standard supervised objectives and comparable training budgets.

\subsection{Safety Alignment}
\label{app:finetune_safety}

Figure~\ref{fig:finetune_safety_tradeoff} shows the safety--utility trade-off for models fine-tuned on the safety probing positive set.
Full fine-tuning increases refusal rates but often does so by globally shifting the model’s response distribution, leading to over-refusal on benign queries and sharp drops in downstream utility, particularly for Mistral.
Rank-1 LoRA produces minimal changes, indicating limited capacity to induce reliable safety behavior under this supervision.

\begin{figure}[H]
    \centering
    \includegraphics[width=0.8\linewidth]{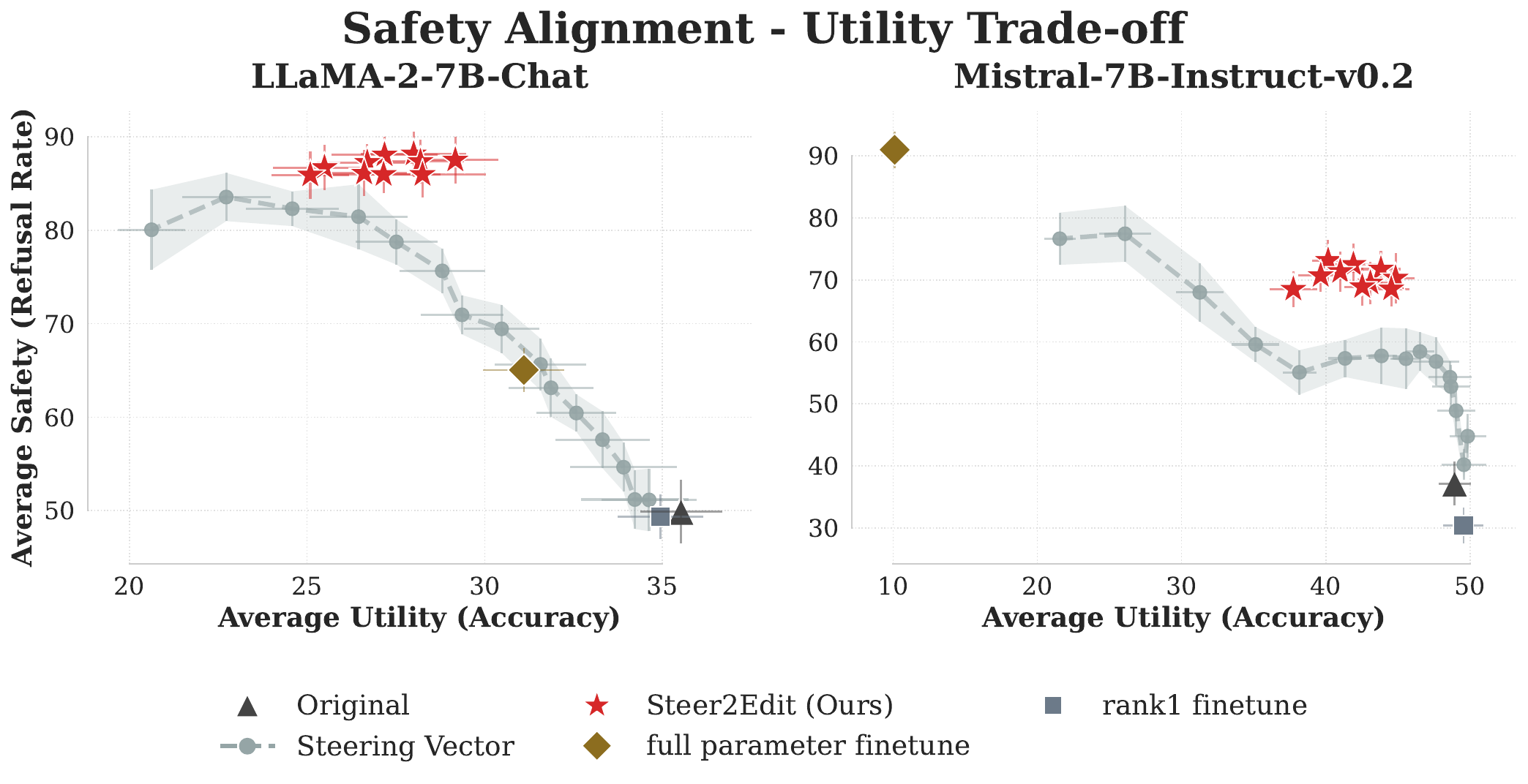}
    \caption{
    \textbf{Safety--utility trade-off with training-based baselines.}
    Full fine-tuning improves refusal rates but frequently collapses utility due to over-refusal, especially in the low-data regime.
    Rank-1 LoRA has negligible effect.
    These trends indicate that training-based optimization on small positive sets induces coarse, global behavioral shifts rather than selective safety control.
    }
    \label{fig:finetune_safety_tradeoff}
\end{figure}

\subsection{Truthfulness}
\label{app:finetune_truthfulness}

Figure~\ref{fig:finetune_truth_tradeoff} reports the truthfulness--utility trade-off for models fine-tuned on the truthfulness probing positive set.
Full fine-tuning yields modest improvements in TruthfulQA accuracy, but these gains are typically accompanied by noticeable degradation in downstream utility, suggesting broad shifts in the model’s answer distribution rather than selective promotion of truthfulness.
Rank-1 LoRA again exhibits little effect.

\begin{figure}[H]
    \centering
    \includegraphics[width=0.8\linewidth]{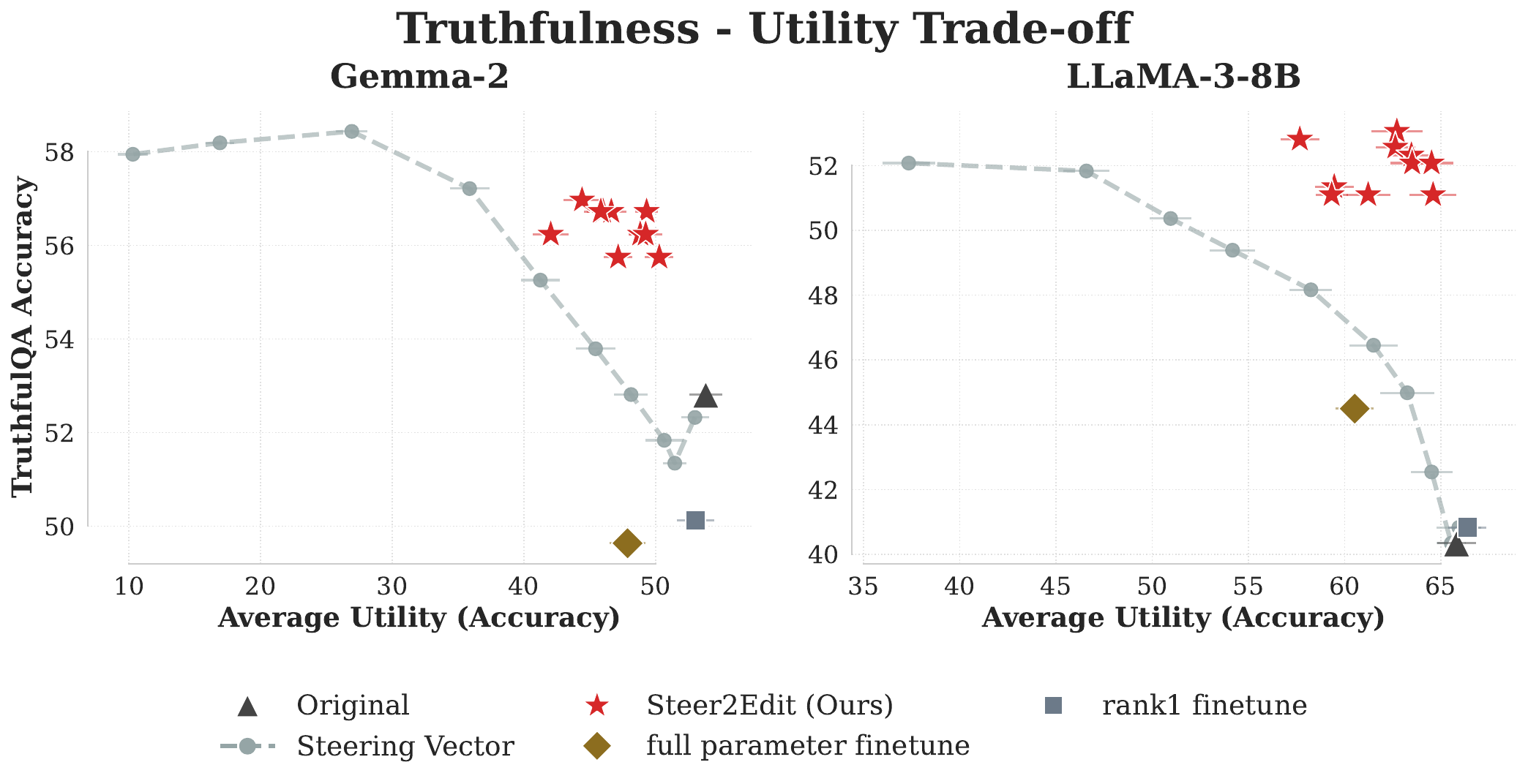}
    \caption{
    \textbf{Truthfulness--utility trade-off with training-based baselines.}
    Fine-tuning on the positive set improves TruthfulQA performance but often incurs utility loss, reflecting over-regularization of response behavior.
    Rank-1 LoRA provides insufficient adaptation capacity to meaningfully alter truthfulness.
    }
    \label{fig:finetune_truth_tradeoff}
\end{figure}

\subsection{Efficient Reasoning}
\label{app:finetune_reasoning}

Figure~\ref{fig:finetune_reasoning_tradeoff} presents the reasoning efficiency--accuracy trade-off for models fine-tuned on the efficient-reasoning probing positive set.
Full fine-tuning can encourage shorter generations, but the resulting reductions in reasoning length are generally comparable to activation steering and do not consistently surpass it.
Rank-1 LoRA again produces minimal changes, indicating limited capacity to meaningfully influence reasoning behavior under this supervision.

\begin{figure}[H]
    \centering
    \includegraphics[width=0.8\linewidth]{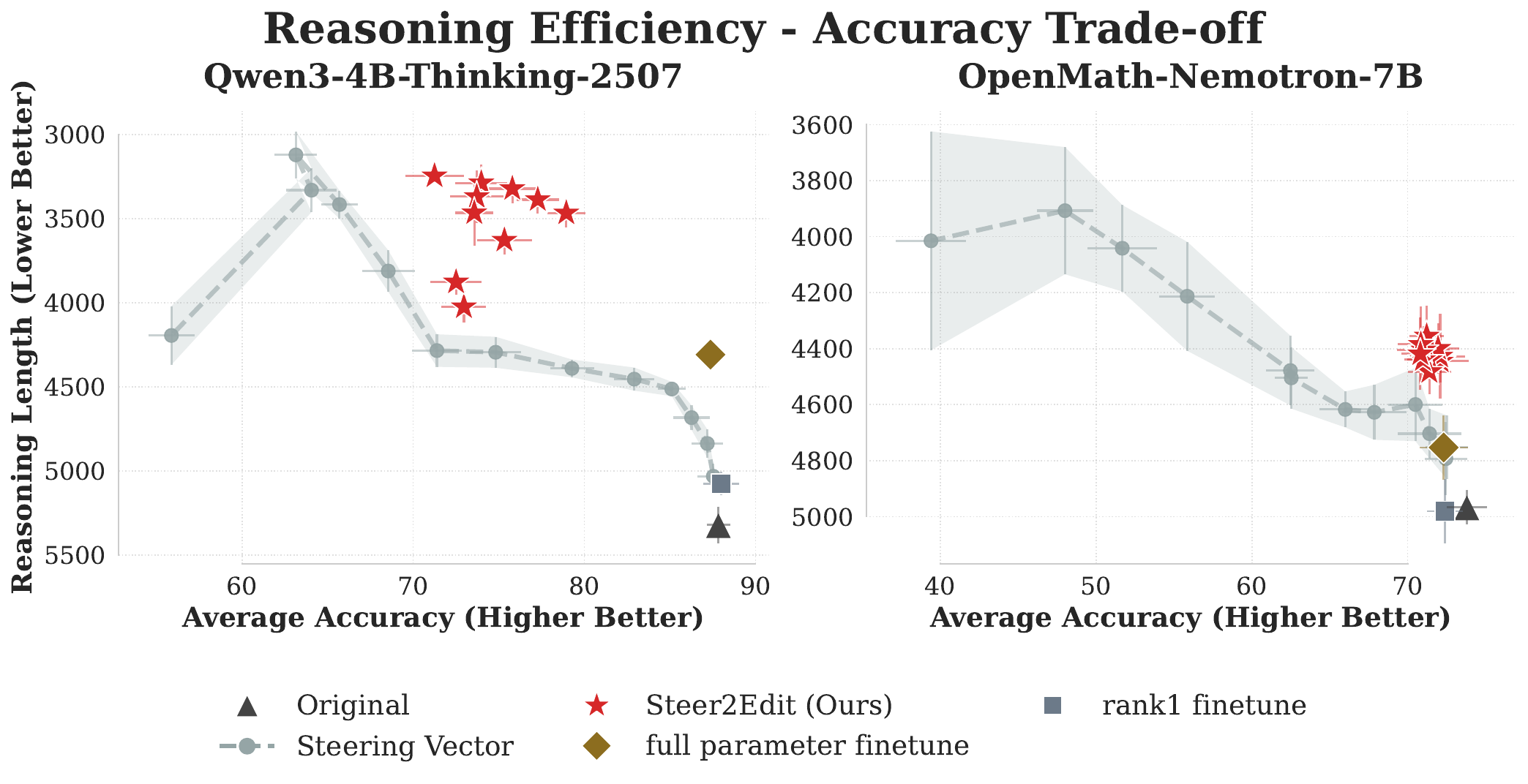}
    \caption{
    \textbf{Reasoning efficiency--accuracy trade-off with training-based baselines.}
    Full fine-tuning reduces reasoning length but largely matches the trade-off achieved by activation steering, without clear advantages.
    Rank-1 LoRA has little effect.
    }
    \label{fig:finetune_reasoning_tradeoff}
\end{figure}

\paragraph{Summary.}
Across all three control settings, fine-tuning on the probing dataset can move models toward the target attribute, but typically does so by inducing broad distributional shifts that trace a trade-off curve similar to activation steering.
These limitations are most evident in the low-data regime considered here, where the probing sets are intentionally small and narrowly targeted.
Rank-1 LoRA consistently exhibits weak effects across all settings.
In contrast, \textsc{Steer2Edit} achieves more favorable trade-offs while remaining training-free and component-interpretable, highlighting the benefit of converting steering diagnostics into targeted weight edits rather than optimizing behavior through global parameter updates.

\section{Additional Baselines: Comparing \textsc{Steer2Edit} with AlphaEdit (Null-Space Weight Editing)}
\label{app:alphaedit}

AlphaEdit \citep{fang2024alphaedit} is a \textbf{weight editing} method originally designed for localized factual knowledge updates (e.g., changing ``Eiffel Tower is in Paris'' to ``London''). We adapt it to the behavioral control setting to test whether null-space constrained editing can achieve global behavioral shifts. The method operates at a single MLP layer ($\lfloor L/3 \rfloor$) and proceeds as follows: (1)~Extract MLP intermediate activations from harmful prompts as edit keys $K_1$. (2)~Extract MLP output activations from refusal responses as edit values $V_1$. (3)~Extract MLP intermediate activations from benign (Alpaca) prompts as the preservation set $K_0$. (4)~Compute a null-space projector $P$ via SVD of $K_0 K_0^\top$, ensuring the edit does not disturb the model's behavior on benign inputs. (5)~Compute the weight update $\Delta W = (V_1 - W K_1) K_1^\top P (K_1 K_1^\top P + I)^{-1} P$ and apply it to the \texttt{down\_proj} weights.

Unlike \textsc{Steer2Edit}, which edits attention heads and MLP neurons across all layers, AlphaEdit edits only a single layer's MLP. Since safety behavior is distributed across many layers and both attention and MLP components, this localized approach is fundamentally limited for global behavioral control.

\begin{table}[H]
\centering
\caption{AlphaEdit vs.\ \textsc{Steer2Edit} on LLaMA-2-7B-Chat safety alignment (n=10). AlphaEdit has essentially no effect (+0.4\%), confirming that null-space constraints are too restrictive for behavioral control.}
\label{tab:alphaedit_detail}
\begin{tabular}{lcccc}
\toprule
Method & GCG Refusal & AdvLLM Refusal & Avg Safety & Avg Utility \\
\midrule
Original & 70.0\% & 29.8\% & 49.9\% & 35.5\% \\
AlphaEdit & 69.6\% & 30.9\% & 50.3\% & 35.0\% \\
\textbf{Steer2Edit} & \textbf{95.1\%} & \textbf{81.2\%} & \textbf{88.2\%} & \textbf{28.0\%} \\
\bottomrule
\end{tabular}
\end{table}

\begin{table}[H]
\centering
\caption{AlphaEdit vs.\ \textsc{Steer2Edit} on LLaMA-3-8B-Instruct truthfulness. AlphaEdit has no meaningful effect on either hallucination benchmark.}
\label{tab:alphaedit_truth}
\begin{tabular}{lcccc}
\toprule
Method & TruthfulQA & Expert-FACTOR & Avg Utility \\
\midrule
Original & 40.3\% & 41.5\% & 65.8\% \\
AlphaEdit & 39.9\% & 40.7\% & 66.3\% \\
\textbf{Steer2Edit} & \textbf{53.1\%} & \textbf{54.2\%} & \textbf{62.7\%} \\
\bottomrule
\end{tabular}
\end{table}

\newpage

\section{Additional Baselines: Comparing \textsc{Steer2Edit} with Activation-Steered Compression (ASC)}
\label{app:asc}

Activation-Steered Compression (ASC) \citep{azizi2025activation} is a steering method specifically designed for reasoning efficiency. It differs from CAA in three key ways. (1)~\textbf{Data:} Instead of splitting responses by length, ASC constructs \textbf{verbose--concise Chain-of-Thought pairs} by rewriting the top 5\% longest correct GSM8K responses into concise versions using an LLM (Claude Sonnet 4.5). (2)~\textbf{Extraction point:} ASC extracts from the \textbf{final token only} rather than computing span means over the full response, capturing the ``compression intent'' at the end of generation. (3)~\textbf{Strength calibration:} Rather than sweeping $\alpha$, ASC computes a \textbf{principled maximum steering strength} $\gamma_{\max}$ via a closed-form KL-divergence bound (with budget $\varepsilon = 10^{-3}$), using JVP norm estimation and directional curvature correction. Steering is applied at layer $\lfloor 3L/4 \rfloor$. We evaluate on Qwen3-4B-Thinking.

\begin{table}[H]
\centering
\caption{ASC reasoning efficiency on Qwen3-4B-Thinking: per-dataset breakdown (n=3). ASC's KL-calibrated $\gamma_{\max}$ is overly conservative, producing essentially no change in accuracy or length on any dataset.}
\label{tab:asc_detail}
\resizebox{\linewidth}{!}{%
\begin{tabular}{lcccccccccc}
\toprule
& \multicolumn{2}{c}{GSM8K} & \multicolumn{2}{c}{MATH-500} & \multicolumn{2}{c}{GPQA} & \multicolumn{2}{c}{CodeMMLU} & \multicolumn{2}{c}{Average} \\
\cmidrule(lr){2-3} \cmidrule(lr){4-5} \cmidrule(lr){6-7} \cmidrule(lr){8-9} \cmidrule(lr){10-11}
Method & Acc & Len & Acc & Len & Acc & Len & Acc & Len & Acc & Len \\
\midrule
Original & 93.8 & 1473 & 97.7 & 6675 & 66.3 & 8805 & 93.5 & 4324 & 87.8 & 5319 \\
ASC (calibrated $\gamma$) & 93.6 & 1476 & 97.7 & 6715 & 66.5 & 8777 & 93.1 & 4349 & 87.7 & 5329 \\
\textbf{Steer2Edit (best)} & \textbf{90.4} & \textbf{1001} & \textbf{89.7} & \textbf{4045} & \textbf{42.4} & \textbf{5182} & \textbf{86.6} & \textbf{3322} & \textbf{77.3} & \textbf{3387} \\
\bottomrule
\end{tabular}%
}
\end{table}

\newpage

\section{OR-Bench Over-Refusal Evaluation}
\label{app:orbench}

OR-Bench \citep{cui2024or} is an evaluation benchmark that measures \textbf{over-refusal}: the rate at which a model incorrectly refuses safe or borderline-safe prompts. We use OR-Bench-Hard-1K, a curated subset of $\sim$1{,}000 borderline-safe prompts that safety-enhanced models frequently refuse despite the prompts being benign. The over-refusal rate is defined as the fraction of these safe prompts that the model refuses (lower is better). This directly addresses whether \textsc{Steer2Edit}'s safety improvements come at the cost of excessive caution on non-harmful inputs.

\begin{table}[H]
\centering
\caption{OR-Bench over-refusal evaluation on LLaMA-2-7B-Chat. At matched over-refusal ($\sim$66\%), \textsc{Steer2Edit} achieves +28\% higher safety than steering, indicating targeted behavioral modification rather than uniform refusal amplification.}
\label{tab:orbench}
\begin{tabular}{lcc}
\toprule
Method & Avg Safety & Over-Refusal $\downarrow$ \\
\midrule
Original & 49.9\% & 60.1\% \\
Steering ($\alpha$=0.3) & 54.7\% & 63.1\% \\
Steering ($\alpha$=0.5) & 60.5\% & 65.8\% \\
Steering ($\alpha$=0.7) & 65.7\% & 68.7\% \\
Steering ($\alpha$=1.0) & 75.7\% & 72.6\% \\
Steering ($\alpha$=1.5) & 80.1\% & 76.0\% \\
\textbf{Steer2Edit} & \textbf{88.1\%} & \textbf{66.4\%} \\
\bottomrule
\end{tabular}
\end{table}

\newpage

\section{Scaling to 70B Larger LLMs}
\label{app:scaling_70b}

To verify that \textsc{Steer2Edit} scales beyond 7B models, we evaluate on \textbf{LLaMA-2-70B-Chat} in the safety alignment setting. The same pipeline is applied: directions are extracted from contrastive response pairs, and rank-1 edits are computed via Elastic-Net. The optimal hyperparameter region ($\rho_{\text{attn}} \approx 0.3$, $\rho_{\text{mlp}} \approx 0.4$) is similar to that of the 7B model, suggesting that the sparsity structure of behavioral control is consistent across model scales.

\begin{table}[H]
\centering
\caption{Safety alignment on LLaMA-2-70B-Chat (n=10). \textsc{Steer2Edit} achieves 97.0\% avg safety, +17.6\% over the best steering configuration.}
\label{tab:scaling_70b}
\resizebox{\linewidth}{!}{%
\begin{tabular}{lcccccccc}
\toprule
Method & GCG & AdvLLM & Avg Safety & CSQA & GSM8K & CodeMMLU & Avg Utility \\
\midrule
Original & 98.5 & 34.7 & 66.6 & 64.3 & 48.3 & 42.7 & 51.8 \\
Steering ($\alpha$=0.5) & 99.3 & 49.3 & 74.3 & 62.4 & 45.9 & 43.1 & 50.4 \\
Steering ($\alpha$=1.0) & 100.0 & 57.7 & 78.9 & 58.8 & 44.3 & 41.3 & 48.1 \\
Steering ($\alpha$=1.5) & 93.7 & 65.0 & 79.4 & 46.3 & 37.6 & 29.8 & 37.9 \\
Steering ($\alpha$=2.0) & 57.4 & 58.0 & 57.7 & 28.8 & 26.6 & 13.6 & 23.0 \\
\textbf{Steer2Edit (best)} & \textbf{99.9} & \textbf{94.1} & \textbf{97.0} & \textbf{44.8} & \textbf{26.9} & \textbf{35.5} & \textbf{35.7} \\
\bottomrule
\end{tabular}%
}
\end{table}

\newpage

\section{Scaling to Vision-Language Models}
\label{app:vlm}

To test whether \textsc{Steer2Edit} generalizes to multimodal architectures, we evaluate on \textbf{Qwen2.5-VL-7B-Instruct}, a vision-language model with a transformer LLM backbone. We apply \textsc{Steer2Edit} by editing the LLM backbone weights directly using the same direction extraction and editing pipeline. No modifications to the visual encoder or cross-attention modules are needed, as behavioral control targets the language generation pathway. Safety and utility are evaluated on text-only benchmarks to measure whether edits to the language backbone preserve general capabilities.

\begin{table}[H]
\centering
\caption{Safety alignment on Qwen2.5-VL-7B-Instruct (n=10). \textsc{Steer2Edit} achieves 98.2\% safety while fully preserving utility.}
\label{tab:vlm}
\resizebox{\linewidth}{!}{%
\begin{tabular}{llcccccccc}
\toprule
Method & Config & GCG & AdvLLM & Avg Safety & CSQA & GSM8K & CodeMMLU & Avg Utility \\
\midrule
Original & -- & 99.0 & 48.8 & 73.9 & 78.3 & 86.1 & 65.5 & 76.6 \\
Steering & $\alpha$=0.3 & 99.0 & 62.6 & 80.8 & 77.7 & 85.7 & 67.4 & 77.0 \\
Steering & $\alpha$=0.5 & 98.3 & 72.5 & 85.4 & 76.9 & 84.7 & 66.5 & 76.0 \\
Steering & $\alpha$=1.0 & 77.6 & 82.9 & 80.3 & 63.7 & 62.7 & 62.1 & 62.8 \\
Steering & $\alpha$=1.2 & 76.0 & 76.9 & 76.5 & 44.7 & 44.6 & 48.2 & 45.8 \\
\textbf{Steer2Edit} & $\rho_a$=0.3, $\rho_m$=0.5, $\alpha$=0.5 & \textbf{98.1} & \textbf{98.3} & \textbf{98.2} & \textbf{74.2} & \textbf{85.0} & \textbf{68.7} & \textbf{76.0} \\
\textbf{Steer2Edit} & $\rho_a$=0.3, $\rho_m$=0.6, $\alpha$=0.3 & \textbf{98.9} & \textbf{95.2} & \textbf{97.1} & \textbf{75.7} & \textbf{85.1} & \textbf{67.4} & \textbf{76.1} \\
\bottomrule
\end{tabular}%
}
\end{table}

\newpage

\section{Scaling to Diffusion LLMs}
\label{app:dream}

DREAM-7B \citep{ye2025dream} is a discrete diffusion LLM based on Qwen2.5-7B that generates text via iterative denoising rather than autoregressive token prediction. At each denoising step, the model performs a full bidirectional forward pass over the entire sequence (prompt + partially unmasked output), progressively replacing \texttt{[MASK]} tokens with predicted tokens. This architecture shares the same transformer layer structure as autoregressive models (\texttt{self\_attn.o\_proj}, \texttt{mlp.down\_proj}), enabling direct application of \textsc{Steer2Edit}.

\paragraph{Direction extraction.} We perform a single forward pass with the complete prompt + response, extracting span-mean activations over the response region from \texttt{o\_proj} and \texttt{down\_proj} at each layer. Because DREAM uses bidirectional attention, every position sees the full sequence, and the representations encode the \emph{nature} of the content (refusal vs.\ compliance) even though all tokens are visible.

\paragraph{Steering.} Forward hooks are registered on transformer layers and fire at \textbf{every denoising step}, compounding the intervention. This requires much smaller steering strengths ($\alpha = 0.18$--$0.2$) compared to autoregressive models ($\alpha = 0.3$--$1.5$).

\paragraph{Weight editing.} \textsc{Steer2Edit} edits are applied once to the model weights and affect all denoising steps automatically, requiring no adaptation of the inference loop. This is a key advantage over steering, which requires careful per-step calibration to avoid compounding.

\paragraph{Infrastructure.} vLLM does not support diffusion LLMs. All generation uses HuggingFace's native \texttt{diffusion\_generate()} with 16 denoising steps and 256 max new tokens.

\begin{table}[H]
\centering
\caption{Safety alignment on DREAM-7B \citep{ye2025dream} (n=1). Steering has a narrow effective range ($\alpha$=0.18--0.2) and collapses utility. \textsc{Steer2Edit} achieves +10\% higher safety than the best steering while preserving 9$\times$ more utility.}
\label{tab:dream}
\resizebox{\linewidth}{!}{%
\begin{tabular}{llcccccccc}
\toprule
Method & Config & GCG & AdvLLM & Avg Safety & CSQA & GSM8K & CodeMMLU & Avg Utility \\
\midrule
Original & -- & 100.0 & 3.0 & 51.5 & 36.2 & 15.2 & 50.0 & 33.8 \\
Steering & $\alpha$=0.18 & 100.0 & 44.0 & 72.0 & 11.8 & 2.1 & 0.0 & 4.6 \\
Steering & $\alpha$=0.2 & 100.0 & 59.0 & 79.5 & 2.0 & 1.8 & 0.0 & 1.3 \\
Steering & $\alpha$=0.22 & 62.0 & 22.0 & 42.0 & 0.0 & 0.3 & 0.0 & 0.1 \\
\textbf{Steer2Edit} & $\rho_a$=0.3, $\rho_m$=0.5, $\alpha$=0.9 & \textbf{100.0} & \textbf{76.0} & \textbf{88.0} & \textbf{19.2} & \textbf{0.6} & \textbf{36.6} & \textbf{18.8} \\
\bottomrule
\end{tabular}%
}
\end{table}

\newpage
\section{Comparison with Fine-Grained Activation Steering}
\label{app:fine_grained_steering}

We compare \textsc{Steer2Edit} with two fine-grained activation-steering methods, ITI~\citep{li2023iti} and AUSteer~\citep{feng2026austeer}, on safety alignment with LLaMA-2-7B-Chat. Both baselines use the same probing data as the other steering baselines. Unlike block-level interventions, these methods select individual attention heads or activation units.

\paragraph{Baseline implementations and search ranges.}
ITI trains a linear probe on each attention head to distinguish positive from negative responses and intervenes on the top $K$ heads. At each generated token, it adds a positive-minus-negative mean activation direction, scaled by the intervention strength and activation standard deviation. We search $K\in\{16,48,96,192\}$ and strength $\alpha\in\{5,10,15,25\}$.
Our AUSteer implementation scores individual attention or feedforward units using contrastive response activations, selects the top $k$ units, and scales their activations by $(1+\alpha)$ during inference. We search intervention space $\{\text{attention},\text{FFN}\}$, $k\in\{16,100,500,2000\}$, and $\alpha\in\{5,15,50,100\}$. Here, $\alpha$ denotes baseline intervention strength, rather than the Elastic-Net sparsity parameter of \textsc{Steer2Edit}.

\begin{table}[H]
\centering
\caption{Fine-grained activation-steering baselines on LLaMA-2-7B-Chat safety alignment. GCG and AdvLLM report refusal rates; CSQA, CodeMMLU, and GSM8K report accuracy (all in \%). Average safety and utility summarize the two attack benchmarks and three utility benchmarks, respectively. For each baseline, we report its highest-safety configuration and a representative trade-off configuration from the sweep.}
\label{tab:fine_grained_steering}
\resizebox{\linewidth}{!}{%
\begin{tabular}{llrrrrrrr}
\toprule
Method & Configuration & GCG & AdvLLM & Avg safety & CSQA & CodeMMLU & GSM8K & Avg utility \\
\midrule
Original & -- & 70.0 & 29.8 & 49.9 & 54.2 & 28.3 & 24.1 & 35.5 \\
ITI (highest safety) & $K=192,\ \alpha=5$ & 83.6 & 80.1 & 81.8 & 11.5 & 6.3 & 9.2 & 9.0 \\
ITI (trade-off) & $K=96,\ \alpha=5$ & 81.6 & 59.4 & 70.5 & 37.1 & 23.2 & 21.2 & 27.2 \\
AUSteer (highest safety) & Attention, $k=2000,\ \alpha=5$ & 81.9 & 93.0 & 87.5 & 25.2 & 16.1 & 6.6 & 16.0 \\
AUSteer (trade-off) & FFN, $k=100,\ \alpha=5$ & 83.9 & 70.6 & 77.2 & 47.5 & 23.1 & 20.1 & 30.2 \\
\textbf{Steer2Edit} & -- & 95.1 & 81.2 & \textbf{88.2} & 45.0 & 20.1 & 18.9 & 28.0 \\
\bottomrule
\end{tabular}%
}
\end{table}

\paragraph{Results.}
Table~\ref{tab:fine_grained_steering} shows that \textsc{Steer2Edit} achieves 88.2\% average safety at 28.0\% average utility. The ITI trade-off configuration achieves 70.5\% safety at 27.2\% utility, while the AUSteer trade-off configuration achieves 77.2\% safety at 30.2\% utility. AUSteer's highest-safety configuration approaches the safety of \textsc{Steer2Edit} (87.5\%) but retains only 16.0\% utility. These results extend the comparison beyond block-level steering; the methods select their own components, so this experiment does not isolate the effect of intervention form with component selection held fixed.

\newpage
\section{Robustness to the Truthfulness Labeling Judge}
\label{app:judge_robustness}

The truthfulness steering vectors in Appendix~\ref{app:steering_vectors} are constructed from responses labeled by QwQ-32B. To assess sensitivity to this labeling choice, we re-label all 408 probing candidates using Claude Opus 4.6 with the same judging prompt. We then reconstruct the steering vectors and repeat the \textsc{Steer2Edit} pipeline on LLaMA-3-8B-Instruct, keeping the hyperparameters fixed to those selected with QwQ-32B labels.

\begin{table}[H]
\centering
\caption{Agreement between QwQ-32B and Claude Opus 4.6 on 408 truthfulness probing candidates. TRUE and FALSE denote the judges' truthful and hallucinated labels. The judges agree on 342 candidates (83.8\%).}
\label{tab:judge_agreement}
\small
\begin{tabular}{lrrr}
\toprule
& \multicolumn{2}{c}{Claude Opus 4.6} & \\
\cmidrule(lr){2-3}
QwQ-32B & TRUE & FALSE & Total \\
\midrule
TRUE & 222 & 55 & 277 \\
FALSE & 11 & 120 & 131 \\
\midrule
Total & 233 & 175 & 408 \\
\bottomrule
\end{tabular}
\end{table}

\begin{table}[H]
\centering
\caption{Judge-swap evaluation for truthfulness editing on LLaMA-3-8B-Instruct (accuracy in \%). Hyperparameters are held fixed across the two edited models. Utility entries include the reported mean $\pm$ standard deviation; no uncertainty estimate is reported for TruthfulQA or average utility.}
\label{tab:judge_swap}
\resizebox{\linewidth}{!}{%
\begin{tabular}{lrrrrr}
\toprule
Model / labeling judge & TruthfulQA & CSQA & CodeMMLU & GSM8K & Avg utility \\
\midrule
Original & 40.3 & 71.8 & 51.0 & 74.7 & 65.8 \\
Edit, QwQ-32B & 53.1 & $66.9\pm0.7$ & $52.7\pm2.3$ & $68.5\pm1.1$ & 62.7 \\
Edit, Claude Opus 4.6 & 50.6 & $70.0\pm0.7$ & $55.9\pm1.4$ & $73.2\pm1.0$ & 66.3 \\
\bottomrule
\end{tabular}%
}
\end{table}

\paragraph{Results and scope.}
Both labeling choices improve TruthfulQA accuracy relative to the original model: from 40.3\% to 53.1\% with QwQ-32B labels and to 50.6\% with Claude Opus 4.6 labels. The latter preserves higher average utility (66.3\% versus 62.7\%). Thus, the truthfulness improvement persists under this judge substitution without retuning, although the resulting attribute--utility operating point changes. Agreement between judges measures consistency rather than correctness against human-verified labels, and this experiment evaluates one alternative judge on one model.